\documentclass[10pt]{article}
\usepackage{aux/parthe_macros}
\usepackage[nonatbib, preprint]{aux/neurips_2022}
\usepackage{enumitem}
\usepackage{subcaption}
\usepackage[noabbrev]{cleveref}
\usepackage{xcolor}
\usepackage{tikz}
\usetikzlibrary{calc,angles,quotes,positioning,patterns}

\newcommand{\loss}{\mathcal{L}}
\newcommand{\xt}{\widetilde{\x}}
\newcommand{\Xt}{\widetilde{\X}}
\newcommand{\pt}{\widetilde{p}}

\newcommand{\Expunder}[1]{\mathbb{E}_{#1}}

\newcommand{\xib}{\boldsymbol{\xi}}

\title{Instability and Local Minima in GAN Training with Kernel Discriminators}
\author{Evan Becker \\
        Dept. CS\\
        UCLA \\
        \texttt{evbecker@cs.ucla.edu}
        \And
        Parthe Pandit \\
        HDSI\\
        UC, San Diego \\
        \texttt{parthepandit@ucsd.edu}
        \And
        Sundeep Rangan \\
        Dept. ECE \\
        NYU\\
        \texttt{srangan@nyu.edu}
        \And
        Alyson K.\ Fletcher \\
        Dept. Statistics \\
        UCLA\\
        \texttt{akfletcher@ucla.edu}
        }
\date{}

\begin{document}
\maketitle

\begin{abstract}
Generative Adversarial Networks (GANs) are a widely-used tool 
for generative modeling of complex data.  
Despite their empirical success, the training of GANs is not fully 
understood due to the min-max optimization of the generator and discriminator.
This paper analyzes these joint dynamics 
when the true samples as well as the generated samples are 
discrete, finite sets,
and the discriminator is kernel-based.
A simple yet expressive framework for analyzing training called the \textit{Isolated Points Model} is introduced. 
In the proposed model, 
the distance between true samples greatly exceeds the kernel width,
so each generated point is influenced by at most one true point.  
Our model enables precise characterization of the  conditions for convergence,
both to good and bad minima.
In particular, the analysis explains two common failure modes: 
(i) an approximate mode collapse and (ii) divergence.
Numerical simulations are provided that predictably replicate these behaviors.
\end{abstract}

\section{Introduction}

Generative Adversarial Networks (GANs) are the most 
widely-used method for learning generative models of complex and structured data in an unsupervised manner. Indeed, GANs have seen incredible 
empirical success in a wide variety of domains ranging from image generation, speech generation, text generation, and many more. Models trained in this manner have also become critical in downstream applications.  
See \cite{wang2021survey, alqahtani_applications_2021} for an overview. 

In the GAN methodology, a \textit{generator} model is trained to output samples that emulate a target dataset, which we call true samples. A critic model, called the \textit{discriminator}, is trained  to tell apart (discriminate) the true and generated samples.  The generator is trained in parallel to fool the discriminator.

Correctly tuning the joint training of the discriminator and generator is one of the key challenges
in GANs and is the source of several empirically observed problematic phenomena.
For example, it is well-known that the resulting distributions can suffer from
mode collapse and catastrophic forgetting.  The optimization can also lead to 
divergence or slow convergence of min-max optimization algorithms. See \cite{mescheder2018training,goodfellow2016nips} for more details. Practical GANs
methods overcome these issues with a combination of careful hyper-parameter optimization
and heuristics.  Significant effort has strived to develop
theoretical frameworks that can better analyze and
optimize GAN training.

In this work, we propose a simple theoretical model, called the \textit{Isolated Points Model}, that is analytically tractable and allows us to rigorously 
study the stability and convergence properties of training a GAN.  
In the proposed model, the true and generated points are discrete distributions
over finite sets, and the discriminator is kernel-based meaning that 
it is linearly parametrized.

We make an additional critical assumption that the true points are sufficiently 
separated such that the kernel interaction between points near two distinct 
true points is negligible.  For distance-based kernels, this assumption essentially 
requires that the true points are separated much greater than the kernel width.
A simple example of this model with four true points 
is illustrated in Fig.~\ref{fig:summary}.

\definecolor{darkgreen}{rgb}{0.5, 0.1, 0.5}
\begin{figure}
\centering
\begin{tikzpicture}[scale=1.0]
    
    \foreach \i/\x/\y in {1/0/4, 2/4/4, 3/0/0,4/4/0} {
        \node [circle,blue,fill=blue!5,draw,
        minimum size=3cm] (c\i) at (\x,\y) {};
        
    }
    
    \foreach \i/\x/\ptrue/\d in {1/-2.3/0.125/0, 2/3/0.25/-0.25,
    3/-2.3/0.375/0.125, 4/3/0.25/0.25} {
        \node [xshift=\x cm,text width=2.5 cm] at (c\i) 
            {$V_{\i}$ \\ $p_{\i}=\ptrue$ \\$\Delta_{\i}=\d$};
    }
    
    \foreach \i/\j/\x/\y in 
        {1/1/-0.8/0.8, 2/3/-0.7/-0.6, 3/3/0.2/0.8,
        4/2/0.2/0.3, 5/2/0/0.3,
        6/2/-0.3/0, 7/2/0.2/-0.2,
        8/4/2/1.2} {
        \node [rectangle,violet,fill=orange,
        draw, minimum width = 0.2cm, 
        minimum height = 0.2cm,
        xshift=\x cm, yshift=\y cm]  (x\i)
        at (c\j) {};
    }
    
    \foreach \i/\j/\x/\y in 
        {1/1/0/0, 2/3/0/0, 3/3/0/0,
        4/2/0.8/0.8, 5/2/-0.8/0.8,
        6/2/-0.8/-0.8, 7/2/0.8/-0.8} {
        \node [rectangle,violet,fill=orange!20,
        draw, minimum width = 0.2cm, 
        minimum height = 0.2cm,
        xshift=\x cm, yshift=\y cm] (xf\i)
        at (c\j) {};
        
        \draw[->,violet]  (x\i) -- (xf\i);
    }
    
    \foreach \i/\r in {1/0.1,2/0.2,3/0.3,4/0.2} {
        \node [circle,draw,olive,fill=green,
            minimum size=\r cm,
            inner sep=0] (xt\i) at (c\i) {};
    }
    
    \node [right of=x8, xshift=1cm, yshift=0cm]
        (xf8) {};
    \node [right of=x8, xshift=0.2cm, 
        yshift=0.5cm] {Diverging};
    \draw[->,violet]  (x8) -- (xf8);
    
\end{tikzpicture}
\caption{(\textbf{Isolated Points Model}) An illustration of our results.
Each isolated region (blue circle) 
has a single true point (green disk) with point mass $p_i$.
There are eight generated points with fixed 
point mass $\pt_j=1/8$ starting at the locations
shown in the orange squares.
In $V_1$ and $V_3$, the excess point mass is non-negative
($\Delta_i \geq 0)$ and the generated points converge to the
true point.  In $V_2$, the excess point mass $\Delta_4 < 0$ and 
the generated points converge to a stable equilibrium around the true point analogous of mode collapse.
Finally, a point outside all four isolated
regions may diverge to $\infty$ in a linear 
velocity.
}
\label{fig:summary}
\end{figure}
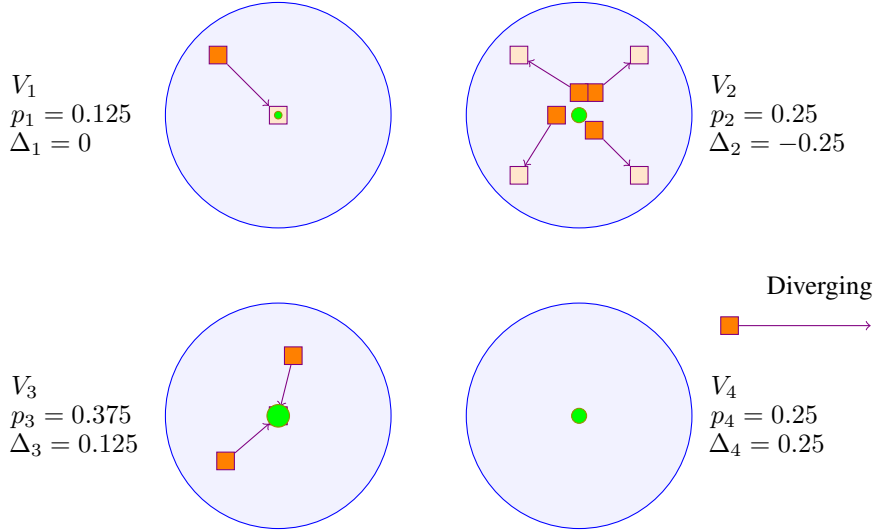

\paragraph{Main Contributions}
We show that this simple isolated points model provides sufficient richness to 
exhibit several interesting phenomena and insights:
\begin{enumerate}[label=(\arabic*), leftmargin=5mm, itemsep=0mm, topsep=0mm]
 
\item \emph{Local stability and instability:} We provide necessary and 
sufficient conditions for local 
convergence of generated points to a true point within each isolated region
(\Cref{thm:linear}). 
The results show that the stability is determined by the excess point mass,
meaning the difference between the true and generated mass in the region.
A consequence of these results, \Cref{cor:support}, is that exact mode collapse where
an excess of generated points concentrate in a single true point, is provably not possible.

\item \emph{Approximate mode collapse:}  Although an excess of generated point mass
cannot exactly concentrate on a single true point, we prove (see \Cref{thm:bad_min}) 
the existence
of locally stable equilibria where an arbitrarily 
large excess point mass concentrates
\emph{near} the true point.  We call this phenomena \emph{approximate
mode collapse}.

\item \emph{Divergence:} We also demonstrate that given a initial perturbation, 
an isolated generated point can lock into an trajectory moving away from any true point in an arbitrary direction. Interestingly, this trajectory is driven solely 
by the generated point's own kernel history.   

\item \emph{Role of the kernel width:} 
The results provide insights into the role of the kernel width 
in training -- See Section~\ref{sec:kernelsel}.
In particular, wider kernel widths reduce the likelihood of different
modes of the distribution being isolated, which is the critical for the 
approximate mode collapse and divergence that we observe.  At the same time, discriminators with
wide kernels make are unable to distinguish points that are close, resulting in slow convergence
near the true distribution.

\end{enumerate}

\paragraph{Prior work}  Convergence problems for GANs have been widely-recognized and studied since their inception~\cite{wang2021survey, alqahtani_applications_2021}.
Indeed, many of the developments in GANs,
notably the popular Wasserstein GAN and its variants,
were motivated to overcome these issues
\cite{arjovsky2017principled, arjovsky2017wasserstein, gulrajani2017improved, Kodali2017dragan}.

For analytic tractability, we focus on a relatively simple GAN with a kernel discriminator and maximum mean discrepancy loss.
This methodology has been applied in other works such as
\cite{franceschi2021neural,binkowski2018demystifying,unterthiner2017coulomb, gretton2012kernel}.
Our focus is on joint optimization of a kernel discriminator and a generator with multiple discrete points.
An important avenue of future work would be to extend
our results to more complex losses such as
the Wasserstein loss \cite{arjovsky2017principled, arjovsky2017wasserstein, gulrajani2017improved}.

The early Dirac-GAN \cite{mescheder2018training},
and its extension \cite{thanh2020forgetting},
considered a linear discriminator and a single Dirac-distribution for the generator and true data.
Our model extends this work by considering
multiple points and a general kernel discriminator. Importantly our model allows for a new parameter called the \textit{excess point mass} which enables modeling complex behaviours in the isolated regions.

Our results also heavily rely on 
linearization methods derived from control theory,
which have been studied in the GAN context in 
\cite{xu2019understanding,nagarajan2018gradient,mroueh2021convergence}. See, also general functional descriptions in
\cite{nie2020towards, khrulkov2021functional, mescheder2017numerics, franceschi2021neural}.
These prior works demonstrated the 
local
stability of the generated distribution close to the true
distribution.  Stability results with a single-layer
NN generator and linear discriminator has been studied 
in \cite{balaji2021understanding}.
Due to the isolated points model, our
results also prove the existence of locally stable \emph{bad} local minima.

Metrics for understanding the performance of generative models remains an open challenge. 
Some efforts in this direction were presented in \cite{Borji2021GANeval, Richardson2018NDB}, and we take advantage of this progress to report our numerical simulations.

\section{Isolated Points Model}\label{sec:model}

We propose the following model for studying the training of GANs.
\paragraph{Discrete true and generated  distributions:}
We assume that the true and generated distributions,
$\Prob_r$ and $\Prob_g$, are over discrete sets in $\Real^d$:
\begin{equation} \label{eq:probrg}
    \Prob_r(\x) = \sum_{i=1}^{N_r} p_i \delta(\x-\x_i), 
    \quad
    \Prob_g(\x) = \sum_{j=1}^{N_g} \pt_j \delta(\x-\xt_j), 
\end{equation}
where $N_r$ and $N_g$ are the number of true and generated points,
$\X = \{\x_i\}_{i=1}^{N_r}$ and $\Xt=\{\xt_j\}_{j=1}^{N_g}$ are the true and generated points, respectively, 
and 
$\{p_i\}$ and $\{\pt_j\}$ are their probabilities.
For the generator, we assume that the probabilities $\pt_j$
are fixed.  
The problem is to learn the point mass locations $\Xt$
so that the generated and true distributions match.

\paragraph*{Kernel discriminator:}
We consider a GAN where the discriminator has a linear parametrization of the form:
\begin{equation} \label{eq:fxtheta}
    f(\x,\thetavec) =  a(\x)^\intercal \thetavec,
\end{equation}
for some vector of basis functions $a(\x)$ and parameter vector $\thetavec$.
The discriminator is trained to maximize a maximum
mean discrepancy (MMD) metric 
\cite{franceschi2021neural,binkowski2018demystifying,unterthiner2017coulomb, gretton2012kernel}
\begin{equation} \label{eq:mmd1}
    \loss_d(\thetavec,\{\xt_j\}_{j=1}^{N_g}) := \sum_{i} p_i f(\x_i,\thetavec) - 
    \sum_{j} \pt_j f(\xt_j,\thetavec) - \frac{\lambda}{2}\|\thetavec\|^2,
\end{equation}
for some regularization parameter $\lambda > 0$.

\paragraph{Discriminator updates:}
We assume simple gradient ascent of the MMD metric 
\eqref{eq:mmd1}:
\begin{equation} \label{eq:thetaup}
    \thetavec^{k+1} = \thetavec^k + \eta_d \rbrac{
    \sum_{i} p_i a(\x_i) - 
    \sum_{j} \pt_j a(\xt_j) - \lambda\thetavec^k },
\end{equation}
where $\eta_d > 0$ is the discriminator step-size.  If we let $f^k(\x) := f(\x,\theta^k)$,
then for any fixed $\x$, we have 
\begin{equation} \label{eq:fup}
    f^{k+1}(\x) = f^k(\x) + \eta_d \rbrac{
    \sum_{i=1}^{N_r} p_i K(\x,\x_i) - 
    \sum_{j=1}^{N_g} \pt_j K(\x,\xt_j) - \lambda f^k(\x)  },
\end{equation}
where $K(\x,\x')$ is the kernel:
\begin{equation} \label{eq:Kgen}
    K(\x,\x') := a(\x)^\intercal a(\x').
\end{equation}
To make the analysis concrete, some of our results 
will apply to the radial basis function (RBF) kernel
\begin{equation} \label{eq:krbf}
    K(\x,\x') = e^{ -\|\x-\x'\|^2/(2\sigma^2)},
\end{equation}
where $\sigma > 0$ is the \emph{kernel width}.

\paragraph{Generator updates:}  We will assume the generator distribution $\Prob_g$
in \eqref{eq:probrg} is directly parameterized by the point mass locations, $\Xt = \{\xt_j\}_{j=1}^{N_g}$,
that are updated to minimize the loss:
\begin{equation} \label{eq:lossg}
    \loss_g(\thetavec,\{\xt_j\}_{j=1}^{N_g}) := -\sum_{j} \pt_j f(\xt_j,\thetavec).
\end{equation}
We consider a simple gradient descent update for this loss
\begin{equation} \label{eq:xtup}
    \xt_j^{k+1} = \xt_j^k + \eta_g\pt_j \nabla f^k(\xt_j).
\end{equation}
Together, \eqref{eq:fup} and \eqref{eq:xtup} define the joint dynamics of the GAN.

\paragraph{Isolated Points:}
Our final key assumption is that the true samples are separated 
far enough so that there exists a non-empty \textit{isolated neighborhood} 
$V_i$ around each sample $\x_i$ such that,
\begin{equation}  \label{eq:Kdist}
    K(\x,\x') = 0 \mbox{ for all } \x \in V_i 
    \mbox{ and } \x' \in V_j \text{ for all $i\neq j$}
\end{equation}
In other words, the samples are separated sufficiently far apart such that they are outside the width of the kernel evaluated at another sample. See \Cref{fig:summary} for an illustration.

The assumption \eqref{eq:Kdist} is obviously strict and
is an idealization of what occurs in practice.  
The supplementary material
discusses modifications of the results to the case
where $|K(\x,\x')|\leq \epsilon$ for some small $\epsilon$
and all $\x \in V_i$ and $\x' \in V_j$.

Now consider an isolated region $V_i$ around a true point $\x_i$. 
At each training step $k$, let $N_i^k$ be the set of indices $j$ such that the generated 
points $\xt_j^k \in V_i$.  We further suppose that $N_i^k$ is constant over time,
so the points $\xt_j^k$ do not enter or exit this region.  Then,
their dynamics within the region are given by:
\begin{subequations} \label{eq:update_local}
\begin{align}
    f^{k+1}(\x) &= f^{k}(\x) + \eta_d\rbrac{ p_i K(\x,\x_i) - \sum_{j \in N_i}
        \pt_j K(\x,\xt_j^{k}) - \lambda f^{k}(\x)}& \forall \x \in V_i,\\
    \xt_j^{k+1} &= \xt_j^{k} + \eta_g \pt_j \nabla f^{k}(\xt_j^{k}) & \forall j \in N_i
\end{align}
\end{subequations}
We will call the updates \eqref{eq:update_local}
the \emph{dynamical system in the region $V_i$}. 

We may also choose to write the updates
\eqref{eq:update_local} in terms
of the components of $\thetavec$.  Let 
$\thetavec_i := \P_i \thetavec$
where $\P_i$ is the projection onto the
range space of $a(\x)$ for $\x \in V_i$.
Then, it can be verified that
the updates in \eqref{eq:update_local}
can be written as:
\begin{subequations} \label{eq:update_local_theta}
\begin{align}
    \thetavec^{k+1}_i &= \thetavec^{k}_i + \eta_d\rbrac{ p_i a(\x_i) - \sum_{j \in N_i}
        \pt_j a(\xt_j^{k}) - \lambda \thetavec^{k}_i(\x)} \\
    \xt_j^{k+1} &= \xt_j^{k} + \eta_g \pt_j \nabla f(\xt_j^{k},\thetavec_i), \quad \forall j \in N_i.
\end{align}
\end{subequations}
We will use the notation
\begin{equation}
    \Xt_i = \left\{ \xt_j, ~ j \in N_i \right\},
\end{equation}
to denote the set of generated points $\xt_j$
in the isolated region.  The state variables for
the dynamics 
\eqref{eq:update_local_theta} can be represented
by the pair $(\thetavec_i, \Xt_i)$. Note that we allow separate learning rates for the two update equations. Our convergence result in \Cref{thm:linear} rely on the ratio of these learning rates.

\section{Behavior Near the True Point}

Consider an isolated region $V_i$ around
some true point $\x_i$.
We first analyze the dynamics where
\begin{equation} \label{eq:xclose}
    \xt_j^k \approx \x_i~ \forall j \in N_i.
\end{equation}
That is, all the generated points in $V_i$ 
are close
to the true point.  We follow a standard
linearization analysis used in several other GAN works
such as 
\cite{gretton2012kernel,li2017mmd,binkowski2018demystifying}.  However, a critical difference in our
model is that the generated mass may not equal the
true mass in a particular isolated region. 
We thus define 
the \emph{probability mass difference}
\begin{equation} \label{eq:deldef}
    \Delta_i := p_i - \sum_{j \in N_i} \pt_j,
\end{equation}
which represents the difference the true probability mass 
in the region, $p_i$, and total probability mass of the
generated points in that region.  Note that this parameter is missing from the analysis of Dirac-GAN \cite{mescheder2017numerics} since $\Delta_i=0$ trivially when only one true point exists. Several numerical behaviors of GANs emerge due to $\Delta_i\neq 0$.

\begin{assumption} \label{as:linear}
The kernel $K(\x,\x')$ is smooth and,
at each true point $\x_i$:
\begin{equation} \label{eq:kgrad_as}
    \left. \frac{\partial K(\x,\x_i)}{\partial \x}
    \right|_{\x=\x_i} = \zero,
\end{equation}
and 
\begin{equation} \label{eq:khess_as}
    \quad
    \left. -\frac{\partial^2 K(\x,\x_i)}{\partial \x^2}
    \right|_{\x=\x_i} \in [k_1,k_2]\I,
    \quad 
    \left. \frac{\partial^2 K(\x,\x')}{\partial \x \partial \x'}
    \right|_{\x=\x'=\x_i} \in [k_3,k_4]\I,
\end{equation}
for some $k_1, k_2,k_3,k_4 > 0$ where we use the
notation that $\Q \in [q_1,q_2]\I$ to mean
$q_1\I \preceq \Q \preceq q_2\I$.
\end{assumption}

The condition is mild and simply requires that the 
kernel $K(\x,\x_i)$ has a local maxima at
$\x = \x_i$ with negative curvature,
and also satisfies a strict positivity condition.  The assumption,
for example, 
is satisfied by the RBF kernel \eqref{eq:krbf}
with $k_1 = k_2 = k_3 =  k_4 = 1/\sigma^2$.

\begin{theorem}  \label{thm:linear}
Fix any isolated region $V_i$, suppose the kernel satisfies \Cref{as:linear} and let 
\begin{equation} \label{eq:xeq}
    \Xt^*_i = \{ \xt^*_j, ~ j \in N_i \},
    \quad 
    \xt^*_j = \x_i.
\end{equation}
Then, there is a unique $\thetavec^*_i$ such that
$(\thetavec^*_i,\Xt^*_i)$
is an equilibrium point of the GAN
dynamics \eqref{eq:update_local_theta}
in the region $V_i$.
In addition, if the parameters $k_i$, $\lambda$ and $\mu=\eta_g/\eta_d$ are fixed, then,
for sufficiently small step size $\eta_d$:
\begin{enumerate}[label=(\alph*),itemsep=0pt,topsep=0pt]
\item If $\Delta_i > 0$,
the equilibrium point is locally stable.
\item If $\Delta_i < 0$ and $|N_i|\geq 1$,
the equilibrium point is locally unstable.
\item If $|N_i|=1$ and 
\begin{equation} \label{eq:sing_stable}
    \mu \Delta_i k_1\pt_1 + \min{} 
    \{ \lambda^2, \mu \pt_1^2 k_3 \} > 0,
\end{equation}
the equilibrium point is locally stable.
\item If $|N_i| = 1$ and
\begin{equation} \label{eq:sing_unstable}
    \mu \Delta_i k_2\pt_1 + \min{} 
    \{ \lambda^2, \mu \pt_1^2 k_4 \} < 0,
\end{equation}
the equilibrium point is locally unstable.
\end{enumerate}
In cases (c) and (d), $\pt_1$ denotes the 
point mass of the single generated point in 
region.
\end{theorem}

The theorem 
provides simple necessary and sufficient
conditions on the stability of equilibrium points at the true point.
In case (a), we see that
when the probability mass difference, $\Delta_i > 0$, the generated points
will converge locally to the true point.  
That is, the generated points will converge to 
the true point as long as the the true point
mass exceeds the total generated mass.
Also, \eqref{eq:sing_stable} will always be 
satisfied when $\Delta_i=0$.  So, when $|N_i|=1$
(i.e., there is a single generated point),
the generated point will converge to the true
point when the generated and true point
have the same mass.  
Examples of these convergence situations are 
shown in \Cref{fig:summary},
in $V_1$ (where $\Delta_1 = 0$) and $V_3$
(where $\Delta_3  > 0$).

Conversely, case (b) shows that 
when $\Delta_i < 0$ (i.e.\ the generated probability exceeds the true probability), the 
generated mass can no longer stably settle in the true point.  This case
is shown in $V_2$ in Fig.~\ref{fig:summary}
where the generated points are repelled from the
true points.

\paragraph{Single vs. multiple points:}  
Note there is a slight difference between the cases when $|N_i|=1$ (i.e.,
there is a single generated point in the region),
and $|N_i|> 1$, when there is more than one generated point.  As the proof of the theorem 
shows, when $|N_i| > 1$, the mean of the generated points may converge to the true point,
but the system may have other modes that
are unstable 
in other directions.

\paragraph{Boundary cases:} Our theorem does not
discuss certain boundary cases.  For example,
when $\Delta_i = 0$ and $|N_i| > 1$, the 
theorem does not state whether the 
equilibrium point is stable or unstable.
These boundary cases are standard 
in linearization
analyses when the eigenvalues of the linearized
system are only critically stable. 
Critically stable linear systems will have 
limit cycles \cite{vidyasagar2002nonlinear}.
The stability 
of general nonlinear systems 
will be determined by the
higher order dynamics, which, in our case,
would be determined by the higher order terms
of the kernel.  However, standard non-linear
systems theory results such as \cite{vidyasagar2002nonlinear} show that
any such points cannot be exponentially stable.
Hence, if the points locally converge, the
convergence rate may be slow.

\paragraph{Relation to Dirac-GAN:} Note that our theorem recovers the stability result \cite[Theorem 4.1]{mescheder2017numerics}
as a special case:  When there is a single true and generated point (whereby $|N_i|=1$ and $\Delta_i=0$),  and regularization is used 
($\lambda > 0$), 
the theorem shows the true point is
locally stable.  When there is no regularization 
$(\lambda = 0)$, the criteria in both \eqref{eq:sing_stable} and \eqref{eq:sing_unstable}
are both exactly zero, meaning the equilibrium
point is not exponentially stable or unstable.
In this case, \cite{mescheder2017numerics}
shows that the system has a limit cycle.

\paragraph{Non-existence of exact mode collapse:} An important corollary of these
results is provided by the following result.

\begin{corollary}  \label{cor:support}
Suppose the kernel 
satisfies \Cref{as:linear}, and the number
of true and generated points are equal
so that $N_r = N_g=N$ for some $N$.
Also, suppose that $p_j = \pt_j=1/N$
for all $j$.  
Then, the only stable equilibrium of the 
generated distribution $\Prob_g$ with 
\begin{equation}
    \mathrm{supp}(\Prob_g) \subseteq \mathrm{supp}(\Prob_r)
\end{equation}
is $\Prob_g = \Prob_r$.
\end{corollary}

In \Cref{cor:support}, $\mathrm{supp}(\cdot)$
denotes the support of the distributions.
So, $\mathrm{supp}(\Prob_r)$ denotes the set of
non-unique values of true points $\x_i$ and 
 and $\mathrm{supp}(\Prob_g)$ 
 denotes the set of
non-unique values of generated points $\xt_j$. 
Thus, the result rules out the possibility 
that the generated distribution can find a locally stable equilibrium where multiple generated points converge
to a single true point.

\section{Bad Local Minima and Approximate Mode Collapse}
\label{sec:bad_min}

Mode collapse is one of the most widely-known
failure mechanisms of GANs, especially for the simple GANs considered in this work \cite{goodfellow2014generative,arjovsky2017principled,mescheder2018training, thanh2020forgetting}. 
\Cref{cor:support} appears to rule out
mode collapse in the sense that there
are no locally stable  equilibria where multiple
generated point arrive in a single true point.
Indeed, \Cref{thm:linear} suggests that
the generated points will be ``repelled'' from
the true point
when the generated mass exceeds the true
mass, i.e., $\Delta_i < 0$.  This repulsive force 
offers the possibility that
the excess 
generated mass can move towards other true points.

However, in this section, we will show 
that even when $\Delta_i < 0$, the generated points may get
stuck close to the true point.
This phenomena is what we call 
\emph{approximate mode collapse}.

\begin{theorem} \label{thm:bad_min}
Fix a region $V_i$ and consider the  dynamical system~\eqref{eq:update_local_theta} 
with the RBF kernel \eqref{eq:krbf}.
Assume $|N_i| > 1$ so there is more than
one generated point in the region.
Then, there exists a constant $c > 0$, 
independent of the kernel width $\sigma$, 
and $N_{\rm max}$, which is only a function of $d$,
such that, if $|N_i| \leq N_{\rm max}$,
the system has a locally stable equilibrium 
with $\Xt_i^* = \{ \xt^*_j, j \in N_i\}$ and
\begin{equation}
    \|\xt^*_j - \x_i \| \leq c\sigma,
\end{equation}
for all $j \in N_i$.
\end{theorem}

The theorem states that, under certain conditions, the generated
points can get stuck in a local minima,
even when the generated mass exceeds the true mass.
An example of this situation is illustrated in \Cref{fig:summary} in the region $V_2$.
The true mass (the green disk at the center
of $V_2$) is itself an unstable location 
for the generated points since $\Delta_i < 0$.  But, the points
may find a stable local minima around the true point.  Moreover, the distance of the local minima
to the $\x_i$ scales with the kernel width $\sigma$.  Hence, for very small $\sigma$,
the generated points with arbitrarily high 
mass may accumulate near a single true point.
This phenomena can thus be seen as a type
of approximate mode collapse.

The proof of \Cref{thm:bad_min} shows
that the bad minima holds for kernels other
than the RBF as well. 

Prior work in mode collapse have identified 
at least two failure mechanisms:
The first is that the discriminator
may have zero gradient on the generated
data when it comes from a low dimensional 
manifold \cite{arjovsky2017principled}.
The above result provides a constructive
proof for the existence of such modes,
and additionally shows that the 
local minima is stable, meaning the gradients
push the generated data locally to the bad
minima.  

A second failure mechanism is catastrophic
forgetting \cite{thanh2020forgetting},
that past generated values are not remembered
by the discriminator.  Since we use a regularized discriminator ($\lambda > 0$),
our discriminator also ``forgets'' past values,
and thus, the existence of the bad minima
would be consistent with catastrophic forgetting.

\section{Divergence}
What happens to generated points isolated 
from all the true points $\{\x_i\}$?
Consider a single generated point $\xt_0$
whose trajectory $\{\xt_0^k\}_{k\geq 0}$ satisfies
\begin{equation} \label{eq:singiso}
    K(\xt^k_0, \xt^k_j) = 0\quad  \forall j \neq 0,
    \qquad 
    K(\xt^k_0,\x_i) = 0\quad \forall i\qquad\qquad\text{for all }k.
\end{equation}
That is, $\xt^k_0$ is sufficiently
far from all the other generated points and true points
so that the kernel can be treated as zero.  In this case,
the dynamics \eqref{eq:fup} and \eqref{eq:xtup}
reduce to
\begin{subequations} \label{eq:update_sing}
\begin{align}
    f^{k+1}(\x) &= (1-\eta_d\lambda)f^k(\x) - K(\x,\xt^k_0) \label{eq:update_singf} \\
    \xt^{k+1}_0 &= \xt^k_0 + \eta_g \nabla f^k(\xt^{k+1}_0). \label{eq:update_singx}
\end{align}
\end{subequations}

\begin{theorem} \label{thm:diverging}
Consider the dynamical system \eqref{eq:update_sing}
with a translation-invariant kernel of the form
$K(\x,\x')= \phi(\|\x-\x'\|)$ for some smooth,
integrable function $\phi(\cdot)$
with $\phi(0) > 0$.
Then, for any initial condition, $\xt^0_0$,
and unit vector $\u \in \Real^d$,
and $\lambda$ sufficiently small,
there exists an initial discriminator function 
$f^0(\x)$ and velocity $v_0 > 0$ such 
that the solution to the system in \cref{eq:update_sing} is
\begin{equation}
    \xt^k_0 = \xt^0_0 + kv_0\u.
\end{equation}
\end{theorem}

Remarkably, the theorem shows that an isolated
generated point can enter a trajectory where it
continues to move linearly in a direction 
simply propelled by its own kernel history $K(\x,\xt_0^k)$.
This situation is illustrated in ~\Cref{fig:summary},
where the generated point in the top right moves to
the right in a straight line.  
We will see several examples of divergence
in the numerical results as well.

\section{Role of the Kernel Width}
\label{sec:kernelsel}

A key parameter of any 
distance based kernels, such as the RBF
\eqref{eq:krbf}, is its \emph{kernel width},
meaning the approximate distance at which 
the kernel begins to decay significantly.
Our results, combined with previous analyses,
provide useful 
insights into
the role of the kernel width in GAN training.
For example, consider the simple case
of the Dirac-GAN where there is a single
true point $\x_0$,
and single generated point $\xt_0$. 
If we fix the generated point, $\xt_0$,
and let an RBF 
discriminator run to the convergence,
it is known from results in 
\cite{gretton2012kernel,li2017mmd,binkowski2018demystifying} that the discriminator 
parameters $\thetavec^k$ will converge
to some values $\thetavec^k \rightarrow \thetavec^*$, and the resulting generator loss 
in \eqref{eq:lossg} is given by
\begin{equation}
    \loss_g(\thetavec^*,\xt_0) = \frac{1}{\lambda}
    \left(1-e^{-\|\x_0-\xt_0\|^2/2\sigma^2} \right).
\end{equation}
The supplementary material reviews these
results in a more general setting.

Now, if the kernel width $\sigma$ is very large,
the loss will have a small gradient when 
$\xt_0$ is close to $\x_0$, i.e., 
the generated point is close to the true value.
On other hand, if $\sigma$ is very small,
the loss will have a small gradient when 
$\xt_0$ is far from $\x_0$.  The selection
of the appropriate $\sigma$ must balance the convergence rates in these two regimes.
In addition, when $\sigma$ is large,
the discriminator is approximately linear
and therefore cannot distinguish nonlinear
regions.

The results for the isolated points model
provide additional insight.  The isolated points
assumption
arises when the true points become separated 
at a distance much greater than $\sigma$.
In such cases, we have seen that
at least two failure modes become possible:
(1) approximate mode collapse, as
described in \Cref{thm:bad_min};
and (2) divergence, as described in 
\Cref{thm:diverging}, where generated 
points isolated from all true points diverge
in arbitrary directions.  
These two failure mechanisms may be prevented
with a wider kernel width to preclude the isolated
points assumption.  However, wide kernel 
widths come at the price of poor discriminator
power and slow convergence near the true point.

\section{Numerical Examples} \label{sec:numerical}
We conduct a simple 
experiment on low-dimensional, synthetic 
data to illustrate the behavior that the theory predicts.  
We also look at the frequency of GAN failure modes as a function of kernel width of the discriminator. Our main observation is that \emph{failure most often occurs when the model is in an isolated points regime} at small kernel width.

Two simple datasets for the true data are used: 
A set of $N_r=4$ points arranged on 
uniformly on the unit circle in dimension $d=2$; and a set of $N_r=10$ 
points randomly distributed on the unit sphere in 
dimension $d=10$.  In both cases, we initialize $N_g=N_r$ generated point as Gaussians with 
zero mean and $\Exp\|\xt_j\|^2=1$.
We approximate the RBF discriminator using a random Fourier feature map as in \cite{Rahimi2007RFF}.  We set $\lambda=0.01$
and $\eta_d=\eta_g=10^{-3}$ and use 40000
training steps.  Other details are in the 
Supplementary material.

\begin{figure}
    \centering
    \includegraphics[width=\linewidth]{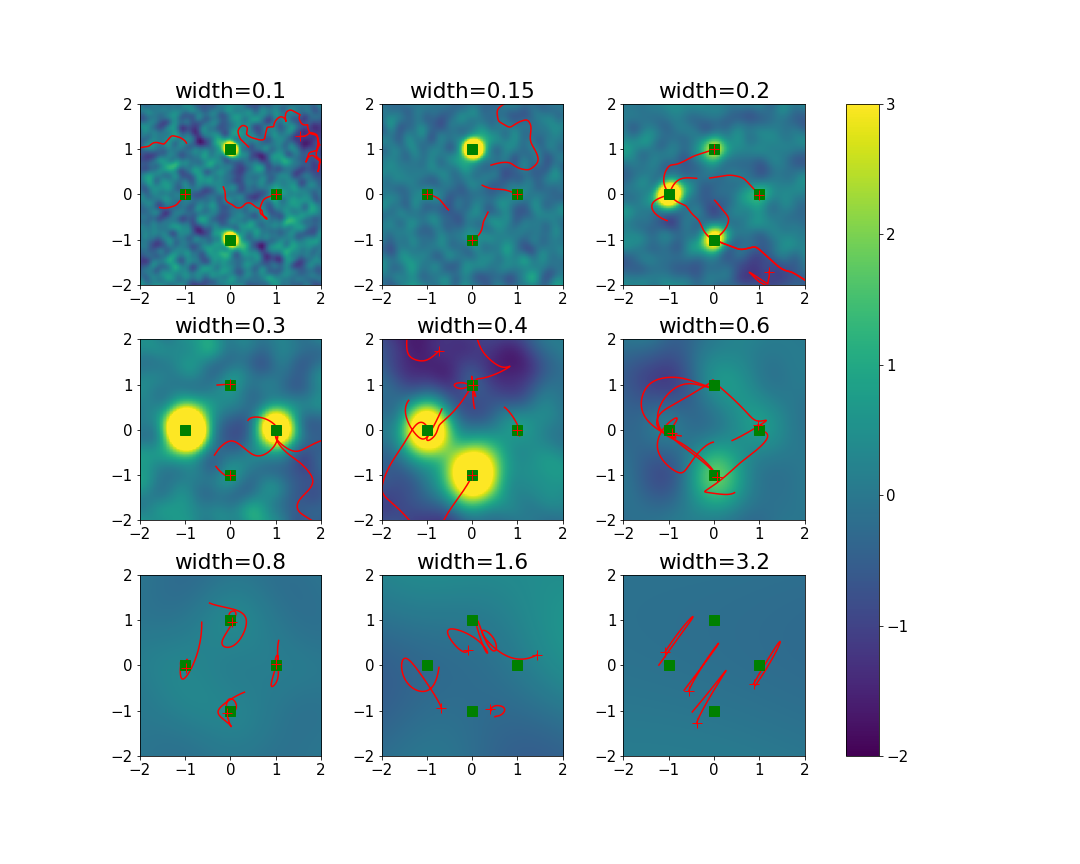}
    
    \vspace{-1cm}
    \caption{\textbf{Behavior of joint GAN training with kernel width}: Example trajectories of generated points over the course of training (red
    lines with final point marked as a cross), true distribution (green), final discriminator (blue and yellow colormap).}
    \label{fig:trajectories}
\end{figure}

\Cref{fig:trajectories} show example
trajectories for the $d=2$ case for different
kernel widths $\sigma$.
For the two dimensional setting we visualize the trajectories of generated points over the course of training in \Cref{fig:trajectories}. We see a diverging behavior at very small kernel widths, as these generated points are not influenced by the kernels of the true points. Instead they ``wander'' off.  Although \Cref{thm:diverging} shows linear
diverging trajectories, other diverging trajectories may be possible due to the randomness
in Fourier feature map.
Also, in the example at $\sigma = 0.6$,
we see an approximate mode collapse where
two generated points converge to a single true point.
Finally, 
at very large kernel widths the discriminator is not able to properly distinguish between distributions, and as a result we see large oscillations. The discriminator gradients are much smoother, meaning they are closer to the linear discriminator regime from the original Dirac-GAN analysis \cite{mescheder2018training,thanh2020forgetting}.

\begin{figure}
\begin{subfigure}[t]{0.48\textwidth}
\centering
    \includegraphics[width=\linewidth]{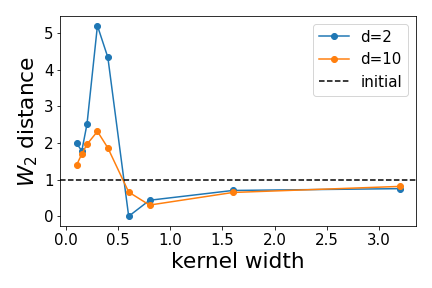}
    \caption{Normalized Wasserstein distance}
    \label{fig:wass-convergence}
\end{subfigure}\hfill
\begin{subfigure}[t]{0.48\textwidth}
    \centering
    \includegraphics[width=\linewidth]{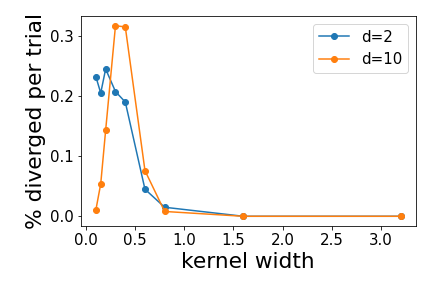}
    \caption{Frequency of divergence}
    \label{fig:freq-divergence}
\end{subfigure}
    \caption{Frequency of different GAN failure modes in a four point, two dimensional setting and a 10 point, 10 dimensional setting.
    (a) The median change in Wasserstein distance between true and generated distributions after 40k iterations. Different kernel widths of an RBF kernel discriminator are tested.
    (b) The average percentage of diverging generated points after training a GAN with RBF kernel discriminator. A generated point is considered diverging if $\| \xt \|_2 > 2 $}
    \label{fig:failures-modes}
\end{figure}

To measure how these failure mechanisms
affect the overall performance,
for each kernel width $\sigma$, we run 100 trials of the true 
and generated points.
\Cref{fig:failures-modes} plots the median
performance of the GAN along two metrics.
In \Cref{fig:failures-modes}(a), 
we measure the Wasserstein-2 distance
\cite{heusel2017gans} between true and generated distributions before and after training period for a fixed number of iterations. The Wasserstein-2
distance is estimated by \cite{bonneel2011displacement}
and provides a error due to diverging points
and mode collapse.  We see that at very low
kernel widths, the normalized error can be
even greater than one, meaning that the generated
points are further from the true points
than their initial value.  This behavior
is a result of the divergence.  At very high
kernel widths, there is also significant error
due to the slow convergence and lack of power
of the discriminator.

As a simple measure of frequency of divergence, \Cref{fig:freq-divergence} measures the average fraction of generated points, $\xt_j$, where  $\|\xt\|_2 > 2$ at the end of the iterations.
Since the true points are on the unit sphere
($\|\x_i\|=1$), this condition can be considered
as a case where the generated points have move
significantly from all the true points.
As expected, we see that the frequency of divergence increases with small kernel width.
However, for dimension $d=10$, 
the frequency of divergence is small for very low
kernel widths.  This phenomena is a result
of the fact that we are approximating the RBF
with random features.  As shown in 
\cite{Rahimi2007RFF}, the number of features required for a good approximation to the RBF grows with $d^2$.

\section{Conclusion}
In this paper we propose a new framework for understanding behaviors exhibited by training GANs. By assuming isolated neighborhoods around each true point, we can decouple the kernel GAN training dynamics that occur when jointly updating a generator and discriminator. 
Using this model we theoretically derive conditions that explain well-known training behavior in GANs. While a stable local equilibrium exists when true and generated distributions match exactly, bad local minima exist as well. Specifically a true point can ``greedily'' hold onto multiple generated points in a region determined by the discriminator kernel width, leading to the phenomenon of approximate mode collapse. We provide theoretical and empirical evidence of a diverging failure mode as well, where one or more generated points completely escape the influence of the true distribution and travel along arbitrary trajectories.
Through this analysis, it becomes clear that kernel width in particular plays an important role in this behavior.

There are several possible lines of future work.
Most importantly, we have studied a simple GAN
without gradient penalties, as is commonly used in
methods such as \cite{arjovsky2017principled, arjovsky2017wasserstein, gulrajani2017improved, Kodali2017dragan}.
Future work can also consider the convergence.
Our linearized analysis provides the eigenvalues
which can be used for rates close to the equilibrium as performed in 
\cite{xu2019understanding,nagarajan2018gradient,mroueh2021convergence}.
Finally, there is a large body of literature
connecting kernels to the neural networks
via the so-called neural tangent kernel (NTK)
\cite{jacot2018neural}
and may be useful to study here as well.

\begin{ack}
S. Rangan was supported in part by NSF grants
1952180, 1925079, 1564142, 1547332, the SRC, and the industrial affiliates of NYU WIRELESS.
\end{ack}

\bibliography{aux/ref}
\bibliographystyle{plain}

\appendix
{\Large\bf Appendix}
\section{Review of Local Stability Results}
\label{sec:dynamical}

We provide a brief review of standard 
definitions and 
stability results that we  use in the theorem
statements and proofs.
All the material can be found in any text in 
nonlinear systems such 
as \cite{vidyasagar2002nonlinear}.
A (discrete-time) dynamical system is simply a recursion
of the form
\begin{equation} \label{eq:dyn_gen}
    \z^{k+1} = \Phi(\z^k), \quad \z^0=\z_0,
\end{equation}
for some (possibly nonlinear) mapping $\Phi(\cdot)$.
Here, $\z_0$, is the initial condition.
A point $\z^*$ is called an \emph{equilibrium point}
of \eqref{eq:dyn_gen} if 
\begin{equation}
    \z^* = \Phi(\z^*).
\end{equation}
The importance of an equilibrium point is that if
a dynamical system is initialized at an
equilibrium point, $\z^0=\z^*$, then it will remain there:  $\z^k= \z^*$ for all $k \geq 0$.

An equilibrium point $\z^*$ is said to be \emph{(locally) stable} if, given any $\epsilon > 0$, there exists a $\delta > 0$
such that 
\[
    \|\z^0 - \z^*\|< \delta \Rightarrow \|z^k - \z^*\|
    < \epsilon \mbox{ for all } k.
\]
That is, the system can remain arbitrarily close to the equilibrium point if it starts sufficiently close.
A system is \emph{unstable} if it is not stable.  An equilibrium point
is \emph{asymptotically stable} if it is stable
and $\delta > 0$ can be chosen such that
\[
    \|\z^0 - \z^*\|< \delta \Rightarrow \lim_{k \rightarrow \infty } \z^k = \z^*.
\]
That is, the $\z^k$ will converge to $\z^*$.
The system is \emph{exponentially stable}, 
if there exists a $\delta > 0$, $c \geq 1$, and 
$\rho \in [0,1)$ such that
\[
    \|\z^0 - \z^*\|< \delta \Rightarrow \|\z^k = \z^*\|\leq c \rho^k \mbox{ for all } k.
\]
Clearly, exponentially stable $\Rightarrow$ asymptotically stable $\Rightarrow$ stable.

The system \eqref{eq:dyn_gen} is \emph{linear} if $\Phi(\z)=\A\z$
for some matrix $\A$.  For a linear system, $\z^*=\zero$
is always an equilibrium point.  The stability  of
 $\z^*=\zero$
is completely determined by the eigenvalues of $\A$.
Specifically, if we let $\mathrm{spec}(\A)$
to denote the set of eigenvalues of a matrix $\A$,
then we have the following well-known result:

\begin{lemma}[\cite{vidyasagar2002nonlinear}]  For a
linear dynamical system \eqref{eq:dyn_gen} with 
$\Phi(\z)=\A\z$ for some matrix $\A$:
\end{lemma}
\begin{enumerate}[label=(\alph*)]
    \item $\z^*=\zero$ is exponentially stable iff $|\rho| < 1$ for all $\rho \in \mathrm{spec}(\A)$.
    \item $\z^*=\zero$ is stable iff $|\rho| \leq 1$ 
    for all $\rho \in \mathrm{spec}(\A)$.
    \item If there exists a single $\rho \in 
    \mathrm{spec}(\A)$ with $|\rho| > 1$, then
    $\z^*=\zero$ is  unstable.
\end{enumerate}

For non-linear systems, the local stability
can be determined by the 
eigenvalues of the Jacobian of $\Phi(\z^*)$, 
which we will denote by $\Gamma(\z^*)$:
\begin{equation} \label{eq:Gamdef}
    \Gamma(\z^*) := \mathrm{spec}\left( \frac{\partial \Phi(\z^*)}
    {\partial \z} \right).
\end{equation}

\begin{lemma}[\cite{vidyasagar2002nonlinear}]
\label{lem:evals_gen}
Consider a dynamical system \eqref{eq:dyn_gen}
for some smooth function $\Phi(\z)$.  
Let $\z^*$ be a equilibrium point,
define $\Gamma(\z^*)$ as 
in \eqref{eq:Gamdef}, 
the spectrum of the Jacobian of $\Phi(\z)$ at $\z=\z^*$.
Then:
\begin{enumerate}[label=(\alph*)]
    \item $\z^*$ is exponentially stable iff $|\rho| < 1$ for all $\rho \in \Gamma(\z^*)$.
    \item If there exists a single $\rho \in 
    \Gamma(\z^*)$ with $|\rho| > 1$, then
    $\z^*$ is locally unstable.
\end{enumerate}
\end{lemma}

A key difference with non-linear systems
is that when the max modulus eigenvalue is 
on the unit circle (i.e., $\max{} |\rho|=1$),
then system may be unstable or stable.
The stability will be determined by higher-order
terms.

\section{Restricted MMD Distance}
\label{sec:local_cost}

As preparation for the proofs, we next introduce
a key function that we will call the restricted
maximum mean discrepancy (MMD) distance.  
This concept is a specialization of the MMD distnace in 
\cite{gretton2012kernel,li2017mmd,binkowski2018demystifying} for the local dynamical system \eqref{eq:update_local_theta}.

Fix the generator locations
\begin{equation}
    \Xt = \left\{ \xt_j, j=1,\ldots,N_g \right\},
\end{equation}
and consider the corresponding generated 
distribution $\Prob_g = \sum_j \pt_j \delta(\x-\xt_j)$.
Suppose we run the discriminator update 
in \eqref{eq:thetaup} to convergence so that
$\thetavec^k \rightarrow \thetavec_i^*$
for some $\thetavec^*$.
A well-known result of GANs is that the resulting
generator loss \eqref{eq:lossg} is given by
\begin{align} \label{eq:mmd}
    \loss_g(\thetavec^*,\Xt) &=
    \frac{1}{2\lambda}\|\Prob_r-\Prob_g\|^2_K,
\end{align}
where the norm is 
the so-called squared \emph{kernel
maximum mean discrepancy} (MMD) distance
\cite{gretton2012kernel,li2017mmd,binkowski2018demystifying}
\begin{equation}
    \|\Prob_r-\Prob_g\|^2_K =
    \Expunder{\x,\x' \sim \Prob_r}{K(\x,\x')}
        -2\Expunder{\x\sim \Prob_r,\xt\sim \Prob_g}{K(\x,\xt)}
        + \Expunder{\xt,\xt' \sim \Prob_r}{K(\xt,\xt')}
\end{equation}
For the discrete distributions \eqref{eq:probrg},
the squared MMD distance \eqref{eq:mmd} can 
be written as a function of the point mass locations
\begin{equation}
    \|\Prob_r-\Prob_g\|^2_K = 2J(\Xt),
\end{equation}
where 
\begin{equation} \label{eq:Jtotal}
    J(\Xt) := 
        \frac{1}{2}\sum_{i,k} p_ip_k K(\x_i,\x_k)
        -\sum_{i,j} p_i\pt_j K(\x_i,\xt_j)
        + \frac{1}{2}\sum_{j,k} \pt_j\pt_k K(\x_j,\xt_k).
\end{equation}
In \eqref{eq:Jtotal}, 
we have omitted the dependence
on the locations of the true distributions $\x_i$
as well as the true and generated weights, $p_i$
and $\pt_j$, since these are fixed in our model.

To analyze the dynamics in a isolated region
$V_i$, we define the \emph{restricted squared
kernel MMD distance} as the function
\begin{equation} \label{eq:Jlocal}
    J_i(\Xt_i) := \frac{1}{2}p_i^2 K(\x_i,\x_i) 
    - \sum_{j \in N_i} p_i\pt_j  K(\x_i,\xt_j) + \frac{1}{2}
    \sum_{j,k \in N_i} \pt_j\pt_k K(\xt_j,\xt_k).
\end{equation}
This function is the squared kernel MMD distance 
\eqref{eq:Jtotal}, 
but only containing the terms with
the single true point $\x_i$
and the set of generated points $\Xt_i = 
\{\xt_j, j \in N_i\}$ in the isolated region $V_i$.

Similar to the MMD analysis in
\cite{gretton2012kernel,li2017mmd,binkowski2018demystifying},
we show that the critical points of the restricted kernel squared
MMD distance are equilibrium points of the local
dynamics \eqref{eq:update_local_theta}.  
To state the result, 
let $\Xt^*_i =\{\xt^*_j, j \in N_i\}$ be a critical point of $J_i(\Xt_i)$
in \eqref{eq:Jlocal} meaning
\begin{equation} \label{eq:Jicrit}
    \left. \frac{\partial J_i(\Xt_i)}{\partial \xt_j}
    \right|_{\Xt_i=\Xt^*_i}= \zero,
    \quad \mbox{ for all } j \in N_i.
\end{equation}
Let
\begin{subequations} \label{eq:thfbad}
\begin{align}
    \thetavec_i^* &:= 
    \frac{1}{\lambda}
    \left[ p_i a(\x_i) - \sum_j \pt_j a(\xt_j^*) \right] \label{eq:thbad} \\
    f^*(\x) &:= f(\x,\thetavec_i^*) = \frac{1}{\lambda}
    \left[ p_i K(\x,\x_i) - \sum_j \pt_j K(\x,\xt_j^*) \right]. \label{eq:fbad}
\end{align}
\end{subequations}
The following is similar to the results in
 \cite{gretton2012kernel,li2017mmd,binkowski2018demystifying}, but applied to the restricted
 squared MMD distance.
 
\begin{lemma} \label{lem:eqpt} 
Let $\Xt^*_i$ be a critical point
of the restricted squared 
MMD distance $J_i(\Xt_i)$ for some $i$.
That is, $\Xt^*_i$ satisfies \eqref{eq:Jicrit}.
Define $\thetavec_i^*$ as in \eqref{eq:thbad}.
Then, the pair $(\thetavec_i^*,\Xt^*_i)$
is an equilibrium point
of the dynamics \eqref{eq:update_local_theta} in the isolated region $V_i$.
Conversely, if $(\thetavec_i^*,\Xt^*_i)$ is an 
equilibrium point of the dynamics \eqref{eq:update_local_theta},
then $\Xt^*_i$ is a critical point of $J_i(\Xt_i)$.
In addition, at any  critical point, $\Xt^*_i$,
of $J_i(\Xt_i)$,
\begin{equation} \label{eq:JHess}
    \left. \frac{\partial^2 J_i(\Xt_i)}{\partial \xt_j^2}\right|_{\Xt_i=\Xt^*_i} 
    = -\lambda H(\xt^*_j,\thetavec_i^*),
\end{equation}
where $H(\xt,\thetavec_i)$ is the Hessian of the 
discriminator
\begin{equation} \label{eq:Hesdef}
    H(\x,\thetavec_i) := 
    \frac{\partial^2 f(\x,\thetavec_i)}{\partial \x^2}.
\end{equation}
\end{lemma}
\begin{proof}
First suppose that $\Xt^*_i$ 
is a critical point of $J_i(\Xt_i)$ and $\thetavec_i^*$
as in \eqref{eq:thbad}.
We need to show that 
$(\thetavec_i^*,\Xt^*_i)$ are fixed points of 
\eqref{eq:update_local_theta}.  That is,
we need to show:
\begin{subequations} \label{eq:update_fix}
\begin{align}
     & p_i a(\x_i) - \sum_{j \in N_i}
        \pt_j a(\xt_j^{*}) - \lambda \thetavec_i^{*} =
        \zero  \label{eq:thfix} \\
     & \nabla f^{*}(\xt_j^{*}) = \zero.
     \label{eq:xtfix} 
\end{align}
\end{subequations}
From \eqref{eq:thbad}, we have
\[
    p_i a(\x_i) - \sum_j \pt_j a(\xt_j^*) - \lambda
    \thetavec_i^* = \zero,
\]
which proves \eqref{eq:thfix}.  
Also, the partial derivative of 
of $J_i(\Xt_i)$ in \eqref{eq:Jlocal} is
\begin{equation} \label{eq:Jderiv}
    \frac{\partial J_i(\Xt_i)}{\partial \xt_j}
    = -p_i  \frac{\partial K(\xt_j,\x_i)}{\partial \xt_j} + \sum_j \pt_j 
    \frac{\partial K(\xt_j,\xt_k)}{\partial \xt_j}
    = -\lambda \nabla f^*(\xt_j),
\end{equation}
where, in the last step, we used the definition
of $f^*(\x)$ in \eqref{eq:fbad}.
Since $\Xt^*_i$ is a local minima of $J_i(\Xt_i)$ we
have 
\[
    \nabla f^*(\xt_j^*) = \left. \frac{\partial J_i(\Xt_i)}{\partial \xt_j}\right|_{\xt_j=\xt_j^*} = \zero,
\]
which shows \eqref{eq:xtfix}.  
The converse is proven by reversing the above steps.
That is, if $(\thetavec_i^*,\Xt_i^*)$
are equilibrium points of \eqref{eq:update_local_theta},
then $\Xt_i^*$ is a critical point of $J_i(\Xt_i)$.
In addition,  taking
the derivative of \eqref{eq:Jderiv} shows 
\eqref{eq:JHess}.
\end{proof}

To analyze the stability of the equilibrium point
in \Cref{lem:eqpt}, we now apply the linearization
methods reviewed in \Cref{sec:dynamical}.  
As mentioned 
in the introduction, most of the stability 
results for GANs follow a similar procedure.
To simplify the notation, WLOG assume that
the set of indices $j \in N_i$ are
\begin{equation}
    N_i = \left\{1,\ldots, N\right\},
\end{equation}
so the the set of points $N_i$ are simply the 
first $N$ generated points 
for some $N$.
Let $\z^k$ denote the state variables
\begin{equation} \label{eq:zcomp}
    \z^k := (\thetavec^k_i,\Xt^k_i) = 
    (\thetavec^k_i, \xt^k_{1}, \ldots, \xt^k_{N})
\end{equation}
for the dynamics \eqref{eq:update_local_theta}.
We can write these updates as
\begin{equation}
    \z^{k+1} = \Phi(\z^k),
\end{equation}
for some non-linear function $\Phi(\cdot)$.
Then $\z^* = (\thetavec^*_i, \Xt^*_i)$ is
an equilibrium point of \eqref{eq:update_local_theta}
if and only if $\z^* = \Phi(\z^*)$.

To apply \Cref{lem:evals_gen}, the following 
lemma will allow us to compute the 
eigenvalues of the Jacobian of the linearization.

\begin{lemma} \label{lem:evals}
Let  $\z^* = (\thetavec_i^*,\Xt^*_i)$ be an equilibrium
point as in \Cref{lem:eqpt} and let $\Gamma(\z^*)$
be the spectrum of the Jacobian of the update map.
Then $\rho \in \Gamma(\z^*)$ if and only if
$\rho$ is of the form
\begin{equation} \label{eq:rhoetas}
    \rho = 1 + \eta_d s,
\end{equation}
where $s=-\lambda$ or $s$ 
is a root of the polynomial
\begin{equation} \label{eq:Jacpoly}
    \psi(s) = \mathrm{det}(\D(s)), 
    \quad \D(s) := 
    (s+\lambda)(s\I+\Q) + \R,
\end{equation}
where $\Q$ and $\R$ are the block matrices 
with components
\begin{equation} \label{eq:QRdef}
    \Q_{ij} = -\mu \pt_i H(\xt_j^*,\theta^*)\delta_{ij}, 
    \quad
    \R_{ij} = \mu \pt_i \pt_j \frac{\partial^2}{\partial \x \partial \x'} \left. K(\x,\x') \right|_{\x=\xt_i^*,\x'=\xt_j^*}.
\end{equation}
\end{lemma}
\begin{proof}
The update map $\z^{k+1}=\Phi(\z^k)$
is defined by the equations \eqref{eq:update_local_theta}.  
Let $\F$ be its Jacobian evaluated at $\z = \z^*$:
\begin{equation}
    \F = \frac{\partial \Phi(\z^*)}{\partial \z}.
\end{equation}
Conformal with the components of $\z$ in 
\eqref{eq:zcomp}, the Jacobian is given by
\begin{equation}
    \F = \I + \eta_d \A,
\end{equation}
where 
\begin{equation} \label{eq:Alin}
    \A = \begin{bmatrix}
        A_{00} & A_{01} & \cdots & A_{0N} \\
        A_{10} & A_{10} & \cdots & A_{1N} \\
        A_{N0} & A_{N0} & \cdots & A_{NN} 
        \end{bmatrix}
\end{equation}
and
\begin{subequations}
\begin{align}
    A_{00} &= -\lambda \I,\\
    A_{0j} &= -\pt_j G(\xt_j^*) \\
    A_{j0} &= \mu \pt_j G(\xt_j^*)^\intercal \\
    A_{jk} &= \mu \pt_j 
    H(\xt_j^*,\thetavec_i^*) \delta_{jk},
\end{align}
\end{subequations}
where $G(\x)$ is the gradient of the basis functions
\begin{equation} \label{eq:Gdefpf}
    G(\x) := \frac{\partial a(\x)}{\partial \x},
\end{equation}
and $H(\x,\thetavec_i)$ is the Hessian in 
\eqref{eq:Hesdef}.



The matrix $\A$ in \eqref{eq:Alin}
can in turn be written as
\begin{equation}
    \A = \begin{bmatrix} 
    -\lambda \I_p & -\G\P \\
    \mu \P\G^\intercal & -\Q
    \end{bmatrix} 
\end{equation}
where
\begin{subequations}
\begin{align}
    \P &= \mathrm{diag}(\pt_1 \I_d, \cdots, \pt_N \I_d), \\
    \G &= \left[ G(\xt_1^*), \ldots, G(\xt_N^*) \right], \\
    \Q &= -\mu \mathrm{diag}(\pt_1 H(\xt_1,\thetavec_i^*), \cdots, \pt_N H(\xt_N,\thetavec_i^*)), \label{eq:Qdefpf}
\end{align}
\end{subequations}
and $p$ is the dimension of $\thetavec_i$
and $d$ is the dimension of $\x$.  Note that
the matrix $\P$ and $\Q$ have dimnesions
$Nd \times Nd$.
Hence for any $s$,
\begin{equation}
    s\I-\A = \begin{bmatrix} 
    (s+\lambda) \I_p & \G\P \\
    -\mu \P\G^\intercal & s\I_{Nd}+ \Q
    \end{bmatrix}.
\end{equation}
Using the determinant of the Schur complement,
\begin{align}
    \MoveEqLeft \mathrm{det}(s\I - \A) 
    \stackrel{(a)}{=} \mathrm{det}((s+\lambda)\I_p)
    \mathrm{det}\left(s\I_{Nd} + \Q + \frac{1}{s+\lambda}\R\right) 
    \nonumber \\
    & \stackrel{(b)}{=} (s+\lambda)^{p-Nd}
    \mathrm{det}\left((s+\lambda)(s\I_{Nd} + \Q) 
    + \R\right) 
    \stackrel{(c)}{=} (s+\lambda)^{p-Nd}
    \psi(s),
\end{align}
where in step (a), we define $\R$ as
\begin{equation} \label{eq:Rdefpf}
    \R = \mu \P \G^\intercal \G\P,
\end{equation}
in step (b), we use the property that
$\mathrm{det}(\alpha \M) = \alpha^m \mathrm{det}(\M)$ for any $\M \in \Complex^{m \times m}$; and in step (c), we used the
definition of $\psi(s)$ in \eqref{eq:Jacpoly}.
This proves that the eigenvalues of $\A$
are either $s=-\lambda$ or the roots of 
$\psi(s)$.  Finally, note that the definition
of $\Q$ in \eqref{eq:Qdefpf} agrees with 
$\Q$ in \eqref{eq:QRdef}.  Also, the components
of $\R$ in \eqref{eq:Rdefpf} are
\begin{align}
    \R_{jk} &=
    \mu[\P\G^\intercal\G\P]_{jk} = 
    \mu \pt_j\pt_k G(\xt^*_j)G(\xt^*_k)
    =  \mu \pt_j\pt_k \left. \frac{\partial^2}{\partial \x \partial \x'} K(\x,\x')\right|_{\x=\x^*_j,\x'=\x^*_k},
\end{align}
where in the last step we used the definition 
of $G(\x)$ in \eqref{eq:Gdefpf}
and the fact that the kernel is
$K(\x,\x') = a(\x)^\intercal a(\x')$.
\end{proof}

Combining \Cref{lem:evals_gen} and \Cref{lem:evals},
we obtain the following simple stability test.

\begin{lemma} \label{lem:stability_flow} Let  $\z^* = (\thetavec_i^*,\Xt^*_i)$ be an equilibrium
point as in \Cref{lem:eqpt}.  Consider the real roots, $\alpha$ for $\psi(\alpha)=0$,
where $\psi(\alpha)$ is the polynomial in 
in \eqref{eq:Jacpoly}.
If all for real roots, $\alpha < 0$,
the equilibrium point
is locally stable for sufficiently small $\eta_d$.
Conversely, if there is a single positive
real root, $\alpha > 0$,
the equilibrium point is locally unstable 
for all $\eta_d$ sufficiently small.
\end{lemma}
\begin{proof}
First suppose that all real roots $\alpha$ of $\psi(\alpha)=0$ are negative, i.e., $\alpha < 0$.
Now suppose that 
$s = \alpha + i\beta$ be a, possibly complex, 
root of the $\psi(s)$.  
Then, there exists a vector $\v \neq \zero$
with 
\begin{equation} \label{eq:Dsv}
    \D(s)\v = ((s+\lambda)(s\I + \Q) + \R)\v = \zero,
\end{equation}
whree $\D(s)$ is defined in \eqref{eq:Jacpoly}.
Since  $\Q=\Q^\intercal$ and $\R=\R^\intercal$, it follows that
\begin{equation} \label{eq:Dalpha}
    \D(\alpha)\v = ((\alpha+\lambda)(\alpha \I + \Q) + \R)\v = \zero.
\end{equation}
Hence, $\alpha$ is a root of $\psi(\alpha) = 0$.
From the assumption, $\alpha < 0$.
Hence, the magnitude squared of the
eigenvalue $\rho$ in \eqref{eq:rhoetas} is
\begin{align}
    |\rho|^2 &= (1+\eta_d \alpha)^2 + \eta_d \beta^2
    = 1+2\alpha \eta_d + \eta_d^2(\alpha^2 + \beta^2)
\end{align}
Since $\alpha < 0$, we will have $|\rho|^2 < 1$ for
all $\eta_d$ with
\begin{equation}
    \eta_d < \min{} \frac{-\alpha}{(\alpha^2 + \beta^2)},
\end{equation}
where the minimum is over all roots of $\psi(s)=0$
with $s=\alpha + i\beta$.

Conversely, suppose that there is at least one real
root $\alpha > 0$ with $\psi(\alpha)=0$.
The magnitude squared of the corresponding 
eigenvalue $\rho$ in \eqref{eq:rhoetas} is
\begin{align}
    |\rho|^2 &= (1+\eta_d \alpha)^2 
    = 1+2\alpha \eta_d + \eta_d^2\alpha^2.
\end{align}
We will have $|\rho|>1$ for $\eta_d < 1/\alpha$.
\end{proof}
\section{Proof of \texorpdfstring{\Cref{thm:linear}}{}}
\label{sec:linear_proof}

Taking the derivatives of $J_i(\Xt_i)$ 
in \eqref{eq:Jlocal} at $\xt_j=\x_i$:
\begin{align}
    \MoveEqLeft
    \frac{\partial }{\partial \xt_j}
    \left. J_i(\xt) \right|_{\xt_j=\x_i}
    = -p_i  \left. \frac{\partial K(\x,\x_i)}{\partial \x}\right|_{\x=\x_i} + \sum_j \pt_j 
    \left. \frac{\partial K(\x,\x_i)}{\partial \x}
    \right|_{\x=\x_i} = \zero,
\end{align}
where the final step uses the assumption
\eqref{eq:kgrad_as}.
Thus, the points $\xt_j^*=\x_i$ are critical points of $J_i(\xt)$ and, by \Cref{lem:eqpt}, they
 are equilibrium points of \eqref{eq:update_local_theta}.
 
We next apply \Cref{lem:stability_flow}
to determine the stability of the equilibrium points. 
Corresponding to the equilibrium points
$\xt_j^*=\x_i$, the discriminator in \eqref{eq:fbad}
is
\begin{equation}
    f^*(\x) = \frac{\Delta_i}{\lambda} K(\x,\x_i),
\end{equation}
where $\Delta_i$ is defined in \eqref{eq:deldef}.
Hence, the Hessian of the discriminator in \eqref{eq:Hesdef}
at $\x=\x_i$ is:
\begin{equation}
    H(\x_i,\thetavec_i^*) = 
    \frac{\Delta_i}{\lambda} 
    \frac{\partial^2}{\partial \x^2} 
    \left. K(\x,\x_i) \right|_{\x=\x_i}.
\end{equation}
Since the equilibrium points
are $\xt^*_j = \x_i$, 
the block diagonal
components in \eqref{eq:QRdef} of the matrix 
$\Q$ are:
\begin{align} \label{eq:QdelH}
    \Q_{jj} &= -\mu \pt_j H(\xt^*_j,\thetavec_i^*)
    = -\mu \pt_j H(\x_i,\thetavec_i^*)
  = -\mu \pt_j \frac{\Delta_i}{\lambda} 
    \frac{\partial^2}{\partial \x^2} \left. K(\x,\x_i) \right|_{\x=\x_i}.
\end{align}
Suppose $\Delta_i > 0$, then 
\begin{equation} \label{eq:Qjjbnd}
    \Q_{jj} = -\mu \pt_j \frac{\Delta_i}{\lambda} 
    \frac{\partial^2}{\partial \x^2} 
    \left. K(\x,\x_i) \right|_{\x=\x_i}
    \geq  \frac{\mu \pt_j\Delta_i}{\lambda} k_1 \I,
\end{equation}
where $k_1$ is defined in \Cref{as:linear}.
Also, the components of the matrix
$\R$ in \eqref{eq:QRdef} are
\begin{align} \label{eq:Rjk}
    \R_{jk} 
    =  \mu \pt_j\pt_k \left. \frac{\partial^2}{\partial \x \partial \x'} K(\x,\x')\right|_{\x=\x^*_j,\x'=\x^*_k} 
    =
    \mu \pt_j\pt_k \R_0, 
\end{align}
where
\begin{equation} 
    \R_0 := \left. \frac{\partial^2}{\partial \x \partial \x'} K(\x,\x')
    \right|_{\x=\x_i,\x'=\x_i}.
\end{equation}
From \Cref{as:linear}, we have 
\begin{equation} \label{eq:Rjkbnd}
    \R_0 \geq k_3 \I.
\end{equation}

\paragraph*{\underline{Case $\Delta_i > 0$}.}
It follows from \eqref{eq:Qjjbnd} that $\Q > 0$.
Also, from the definition of $\R$ 
in \eqref{eq:QRdef}, $\R \geq 0$.
Therefore, for any $\alpha \geq 0$,
the matrix $\D(\alpha)$ in \eqref{eq:Jacpoly}
is bounded below by
\[
    \D(\alpha) = (\alpha+\lambda)(\alpha \I 
    + \Q) + \R
    \geq \lambda \Q > 0.
\]
Hence, for $\psi(\alpha)$ in 
\eqref{eq:Jacpoly}, we have
 $\psi(\alpha) \neq 0$.  Thus, $\psi(\alpha)$
has no roots when $\alpha \geq 0$.
From \Cref{lem:stability_flow},
the system is locally stable
for sufficiently small $\mu_d$.
This proves case (a) of \Cref{thm:linear}.

\paragraph*{\underline{Case $\Delta_i < 0$ and $|N_i|\geq 2$}.}
In this case, \eqref{eq:QdelH} and
\Cref{as:linear} shows that
\begin{equation} \label{eq:Qbnd1}
    -\Q \in [q_1,q_2]\I,
\end{equation}
where 
\begin{equation} 
    q_1 = -\frac{\mu\Delta_ik_1}{\lambda}
    \min{j\in N_i} \pt_j,
    \quad
    q_2 = -\frac{\mu\Delta_ik_2}{\lambda}
    \max{j\in N_i} \pt_j.
\end{equation}
Since $\Delta_i < 0$, $q_2 \geq q_1 > 0$.
For $\D(\alpha)$ in \eqref{eq:Jacpoly}
and $\alpha \geq 0$, let
\begin{equation}
    \rho_{\rm min}(\D(\alpha)) = \min{\|\v\|=1}
    \v^\intercal \D(\alpha) \v,
\end{equation}
which is also the minimum eigenvalue of $\D(\alpha)$.  Note that $\rho_{\rm min}(\D(\alpha))$ is continuous in $\alpha$.

Since $\R$ has the components \eqref{eq:Rjk} and $|N_i| \geq 2$, the
matrix $\R$ is rank-deficient.  Therefore,
by selecting any vector $\v$ in the null space
of $\R$, we obtain
\begin{equation}
    \rho_{\rm min}(\D(\alpha)) 
    \geq (\alpha+\lambda)(\alpha-q_2).
\end{equation}
In particular, at $\alpha=0$,
\begin{equation}
    \rho_{\rm min}(\D(0)) = -\lambda q_2 < 0.
\end{equation}
Also, since $\R \geq 0$,
\begin{equation}
    \rho_{\rm min}(\D(\alpha)) 
    \geq (\alpha+\lambda)(\alpha-q_1),
\end{equation}
and
\begin{equation}
    \rho_{\rm min}(\D(\alpha)) > 0,
\end{equation}
for $\alpha > q_1$.  Hence, there must
be an $\alpha \geq 0$ where
$\rho_{\rm min}(\D(\alpha)) = 0$,
which implies that $\psi(\alpha)=0$
where $\psi(\alpha)$ is the polynomial 
in \eqref{eq:Jacpoly}.  It follows that $\psi(\alpha)$ has  root
with $\alpha > 0$ and by \Cref{lem:stability_flow},
the equilibrium point $(\theta^*_i,\Xt^*_i)$ 
is locally unstable for all $\eta_d$ sufficiently
small.
This proves case (b) of \Cref{thm:linear}.

\paragraph{\underline{Case $\Delta_i \leq 0$ and
$|N_i|=1$}.} 
When $|N_i|=1$,  there is a single generated point.   WLOG suppose the single element in $N_i$
is $j=1$.  
In this case,
In this case, \eqref{eq:QdelH} and
\Cref{as:linear} shows that
\begin{equation} \label{eq:Qbnd2}
    -\Q \in [q_1,q_2]\I,
\end{equation}
where 
\begin{equation} \label{eq:Qjjlim1}
    q_1 = -\frac{\mu\Delta_ik_1\pt_1}{\lambda}
    \quad
    q_2 = -\frac{\mu\Delta_ik_2\pt_1}{\lambda}.
\end{equation}
Also,
\eqref{eq:Rjk} shows that
\begin{align}
    \R &= \R_{11} = \mu \pt_1\pt_1 \R_0 \in \mu [r_1,k_2]\I.
\end{align}
where 
\begin{equation} \label{eq:Rlim1}
    r_1 = \mu\pt_1^2k_3, \quad 
    \quad
    r_2 = \mu\pt_1^2k_4.
\end{equation}
Therefore, for 
the matrix $\D(\alpha)$ in \eqref{eq:Jacpoly},
and $\alpha \geq 0$,
\begin{equation}
    \rho_{\rm min}(\D(\alpha)) \geq (\alpha+\lambda)(\alpha - q_2) + r_1
    = \alpha^2 + (\lambda - q_2)\alpha + r_1 - q_2 \lambda.
\end{equation}
This polynomial will have no non-negative roots 
$\alpha \geq 0$, if
\begin{equation}
    \lambda- q_2 > 0 \mbox{ and } 
    r_1 - q_2 \lambda > 0.
\end{equation}
Using \eqref{eq:Qjjlim1} and \eqref{eq:Rlim1},
this condition is equivalent to
\begin{equation}
    \mu \Delta_i k_2\pt_1 + \min{} 
    \{ \lambda^2, \mu \pt_1^2 k_4 \} > 0.
\end{equation}
In this case, $\rho_{\rm min}(\D(\alpha)) > 0$
for all $\alpha \geq 0$ and 
$\psi(\alpha)$ has no non-negative roots.
From \Cref{lem:stability_flow},
the equilibrium point $(\theta^*_i,\Xt^*_i)$ 
is locally stable for all $\eta_d$ sufficiently
small.
This proves case (c) of \Cref{thm:linear}.

Similarly, taking an upper bound:
the matrix $\D(\alpha)$ in \eqref{eq:Jacpoly},
and $\alpha \geq 0$,
\begin{equation}
    \rho_{\rm min}(\D(\alpha)) \leq (\alpha+\lambda)(\alpha - q_1) + r_2
    = \alpha^2 + (\lambda - q_1)\alpha + r_2 - q_1 \lambda.
\end{equation}
This polynomial will have a positive root $\alpha$ if 
\begin{equation}
    \lambda- q_1 > 0 \mbox{ or } 
    r_1 - q_1 \lambda > 0.
\end{equation}
Using \eqref{eq:Qjjlim1} and \eqref{eq:Rlim1},
this condition is equivalent to
\begin{equation}
    \mu \Delta_i k_1\pt_1 + \min{} 
    \{ \lambda^2, \mu \pt_1^2 k_3 \} < 0.
\end{equation}
In this case, $\rho_{\rm min}(\D(\alpha)) = 0$
for some $\alpha > 0$.
From \Cref{lem:stability_flow},
the equilibrium point $(\theta^*_i,\Xt^*_i)$ 
is locally unstable for all $\eta_d$ sufficiently
small.
This proves case (d) of \Cref{thm:linear}.

\section{Proof of \texorpdfstring{\Cref{cor:support}}{}}

First, we show that $\Prob_g=\Prob_r$ is a stable
local equilibrium.  This situation can only
occur when, for each generated point $j$
\begin{equation} \label{eq:xmatch}
 \xt^*_j = x_i \mbox{  and  } \pt_j = p_i,   
\end{equation}
for some $i$.  Moreover, each true point
$\x_i$ must have exactly one generated point
$j$ with \eqref{eq:xmatch}.
Otherwise, there would be at least one true 
point $\x_i$ with no generated points
and $\Prob_r \neq \Prob_g$.
Thus, we have $|N_i|=1$ and $\Delta_i = 0$
for all $i$.  This condition satisfies
\eqref{eq:sing_stable} so the equilibria
are locally stable.

Now consider any equilibrium points $\Xt = \{\xt_j^*\}$ where $\mathrm{supp}(\Prob_g) \subseteq
\mathrm{supp}(\Prob_r)$.
Then, at least one true point $\x_i$
must have more than one generated point, $j$,
with $\xt^*_j = \x_i$.  That is, $|N_i|\geq 2$.
Also, since the point masses are uniform, 
we will have 
\[
    \Delta_i = p_i - \sum_{j \in N_i} \pt_j
    = \frac{1}{N}\left(1 - |N_i|\right) < 0.
\]
From \Cref{thm:linear}, this equilibrium
is not stable.
\section{Proof of \texorpdfstring{\Cref{thm:bad_min}}{} }

We will prove the theorem under somewhat more general assumptions on the kernel $K(\x,\x')$
as described in the following three assumptions.

\begin{assumption}  \label{as:kscale}  The kernel $K(\x,\x')$ satisfies $K(\x,\x') \in [0,1]$ for all $\x,\x'$ with $K(\x,\x)=1$ for all $\x$.
In addition,  $\lim_{\x'} K(\x,\x') = 0$ as $\|\x'\|\rightarrow \infty$.
\end{assumption}

\medskip
The next assumption is somewhat technical,
although its role will be clear in the proof.

\begin{assumption}  \label{as:ksep}  
In an isolated region $V_i$ around the true point
$\x_i$, there exists a set of distinct
generated points $\Xt_i = \{\xt_j, j \in N_i\}$
such that
\begin{equation} \label{eq:ksep}
    \frac{1}{2} \sum_{j \neq k} \pt_j \pt_k K(\xt_j,\xt_k) <  p_i \sum_{j} \pt_j K(\x_i,\xt_j),
\end{equation}
where the summations are over $j,k\in N_i$.
\end{assumption}

\medskip
The final assumption requires a definition.  
Given a set of points
$\Xt= \{\xt_j ,j=1,\ldots,N\}$, let 
$M(\Xt)$ be the matrix with block components
\begin{equation} \label{eq:Ddef}
    M(\Xt)_{ij} = \frac{\partial^2}{\partial \x, \x'} \left. K(\x,\x') \right|_{\x = \xt_i, \x=\xt_j}.
\end{equation}

\begin{assumption} \label{as:kfullrank}
For any finite set of points $\Xt$,
$M(\Xt)$ in \eqref{eq:Ddef} is full rank.
\end{assumption}

\begin{lemma}  \label{lem:bad_min}
Consider the local squared MMD distance $J_i(\Xt_i)$
in \eqref{eq:Jlocal}.
Under 
\Cref{as:kscale} and \Cref{as:ksep}, 
there exists at least one local minima $\Xt_i^*=\{\xt_j^*, j \in N_i\}$
of $J_i(\Xt_i)$ with $\|\xt_j^*-\x_i\| < \infty$ for all $j \in N_i$.  In addition, the
values $\xt_j^*$ are distinct for different $j \in N_i$.
\end{lemma}
\begin{proof}
Using \Cref{as:kscale}, we can rewrite
the the local cost function \eqref{eq:Jlocal} as
\begin{equation} \label{eq:Jlocal1}
    J_i(\Xt_i) := J_0 
    - \sum_{j \in N_i} p_i\pt_j  K(\x_i,\xt_j) + \frac{1}{2}
    \sum_{j \neq k} \pt_j\pt_k K(\xt_j,\xt_k),
\end{equation}
where 
\begin{equation}
    J_0 :=   \frac{1}{2}\left( p_i^2 + \sum_{j \in N_i} \pt_j^2 \right).
\end{equation}
By \Cref{as:ksep}, there exists at least one
$\Xt_i$ such that
\begin{equation} \label{eq:Jeps}
    J_i(\Xt) \leq J_0 - \epsilon,
\end{equation}
for some $\epsilon > 0$.
Now consider any 
limit of points $\xt_j \rightarrow \infty$ for all $j$.
From \eqref{eq:Jlocal1}, we have 
\begin{align} 
    \MoveEqLeft \liminf_{\xt_j} J_i(\Xt_i) \stackrel{(a)}{=}
    J_0 + \frac{1}{2} \liminf_{\xt} \sum_{j \neq k} \pt_j\pt_k K(\xt_j,\xt_k) \stackrel{(b)}{\geq} J_0, \label{eq:Jliminf}
\end{align}
where (a) follows from \Cref{as:kscale} that
$K(\x_i,\xt_j) \rightarrow 0$ and
(b) follows from the fact that $K(\xt_j,\xt_k) \geq 0$ for
all $\xt_j$ and $\xt_j$.  Since \eqref{eq:Jliminf} shows that
$\liminf J(\xt) \geq J_0$ as $\xt_j \rightarrow \infty$ and
\eqref{eq:Jeps} shows that there is a point with $J_i(\Xt_i) < J_0 
- \epsilon$, there must be at least one local minimum with finite
coordinates.
\end{proof}

We now state a more general version of
\Cref{thm:bad_min}.

\begin{theorem} \label{thm:bad_min_gen}
Fix a region $V_i$ and consider the  dynamical system~\eqref{eq:update_local_theta}
with $|N_i| \geq 2$.
If the kernel satisfies Assumptions~
\ref{as:kscale}---\ref{as:kfullrank}.
the dynamical system has at least one
equilibrium with 
with $\Xt_i^* = \{ \xt^*_j, j \in N_i\}$ where
\begin{equation} \label{eq:xtfinite}
    \|\xt^*_j - \x_i \|^2  < \infty,
\end{equation}
for all $j \in N_i$, the $\xt_j^*$ are
distinct for different $j \in N_i$ 
and the equilibrium point
is locally stable for sufficiently small $\eta_d$.
\end{theorem}
\begin{proof}
From \Cref{lem:bad_min}, there exists
a local minimum $\Xt_i^*=\{\xt_j^*\}$
satisfying \eqref{eq:xtfinite}.
From \Cref{lem:eqpt}, there exists
a $\thetavec_i^*$ such that $(\thetavec_i^*,\Xt_i^*)$ is an equilibrium point
of \eqref{eq:update_local_theta}.
So, it remains to show that the equilibrium
point is locally stable.  
Since $\Xt_i^*$ is a local minima of $J_i(\Xt)$
we have
\begin{equation} \label{eq:JHessbad}
    \left. \frac{\partial^2 J_i(\Xt_i)}{\partial \xt_j^2}\right|_{\Xt_i=\Xt^*_i} 
    \geq \zero.
\end{equation}
Hence, from \eqref{eq:JHess}, we have
\begin{equation} \label{eq:Hessbad}
     -H(\xt^*_j,\thetavec_i^*) = 
    \frac{1}{\lambda}\left. \frac{\partial^2 J_i(\Xt_i)}{\partial \xt_j^2}\right|_{\Xt_i=\Xt^*_i} \geq \zero.
\end{equation}
Thus, the matrices $\Q_{jj}$ in \eqref{eq:QRdef}
are positive semi-definite, and we have $\Q \geq \zero$.  Also, since the points $\xt_j^*$
are distinct, \Cref{as:kfullrank}
shows that the matrix $\R$ in \eqref{eq:QRdef}
satisfies $\R > 0$.  Hence, for all $\alpha \geq 0$, the matrix $\D(\alpha)$ in \eqref{eq:Jacpoly}
satisfies
\begin{equation}
    \D(\alpha) = (\alpha+\lambda)(\alpha \I + \Q)
    + \R \geq \R > \zero.
\end{equation}
It follows that 
\[
    \psi(\alpha) = \mathrm{det}(\D(\alpha)) \neq 0,
\]
and $\psi(\alpha)$ has no real non-negative roots.
The theorem now follows from \Cref{lem:stability_flow}.
\end{proof}

We can now prove \Cref{thm:bad_min} 
as a special case of \Cref{thm:bad_min_gen}.

\paragraph*{Proof of \Cref{thm:bad_min}}
To apply \Cref{thm:bad_min_gen},
we first shows the RBF kernel \eqref{eq:krbf}
satisfies Assumptions~
\ref{as:kscale}---\ref{as:kfullrank}.

\underline{\Cref{as:kscale}:}
This assumptions follows directly from the form
of the RBF kernel \eqref{eq:krbf}.

\underline{\Cref{as:ksep}}: 
Given a set $\U = \{ \u_1,\ldots,\u_K \} 
\subset \Real^d$, with $\|\u_j\|=1$ for all $j$,
define 
\begin{equation}
    \rho_{\rm min}(\U) = \max{j \neq k}
    \u_j^\intercal \u_k,
\end{equation}
which is the maximum angle cosine between two
unit vectors in the set.   Select any $\delta < 1/2$
and set
\begin{equation} \label{eq:Addel}
    N_{\rm max} = \max{} |\U| 
    \mbox{ s.t. } \rho_{\rm min}(\U) \leq \delta,
\end{equation}
which is the maximum cardinality of the set
while keeping the angle cosine 
less than $\delta$.
Now assume $|N_i| \leq N_{\rm max}$.
The bound \eqref{eq:Addel}
states that we can find 
 at least $|N_i|$  unit vectors 
$\u_j$, $j=1,\ldots,|N_i|$ such that
\begin{equation}
    \u_j^\intercal \u_k < \delta < \frac{1}{2}
\end{equation}
for all $j \neq k$.
Since $\delta < 1/2$, we find an $r$ such that
\begin{equation} \label{eq:rbndexp}
    \frac{1}{2}e^{-r^2(1-\delta)}\sum_j \pt_j 
     \leq p_i e^{-r^2/2}.
\end{equation}
Take the generated vectors as
\begin{equation}
    \xt_j = \x_i + r\sigma\u_j.
\end{equation}
Then, for the RBF
kernel \eqref{eq:krbf}, we have
\begin{align} \label{eq:krbnd1}
    K(\x_i,\xt_j) = e^{-\|r\sigma\u_j\|^2 /(2\sigma^2)}
    = e^{-r^2/2}.
\end{align}
Also, the distance between any two generated
points $\xt_j$ and $\xt_k$ with $j \neq k$
is
\begin{equation}
    \|\xt_j - \xt_k\|^2 
    = 2\sigma^2r^2(1 - \u_j^\intercal \u_k)
    \geq 2\sigma^2 r^2(1-\delta),
\end{equation}
and hence
\begin{equation} \label{eq:krbnd2}
    K(\xt_j,\xt_k) \leq e^{-r^2(1-\delta)}.
\end{equation}
We can then verify the bound in \eqref{eq:ksep}:
\begin{align}
    \MoveEqLeft
    \frac{1}{2} \sum_{j \neq k} \pt_j \pt_k
    K(\xt_j,\xt_k) \stackrel{(a)}{\leq}
    \frac{1}{2} \sum_{j \neq k} \pt_j \pt_k
    e^{-r^2(1-\delta)} \nonumber \\
    & \stackrel{(b)}{\leq}
    \frac{1}{2} e^{-r^2(1-\delta)} \left(\sum_{j} \pt_j \right)^2 
    \stackrel{(c)}{\leq}
    e^{-r^2/2} p_i\sum_{j} \pt_j 
    \stackrel{(d)}{\leq} p_i\sum_{j} \pt_j K(\x_i,\xt_j),
\end{align}
where (a) follows from \eqref{eq:krbnd2};
(b) follows since we added a positive term;
(c) follows from \eqref{eq:rbndexp};
and (d) follows form \eqref{eq:krbnd1}.
This proves \Cref{as:ksep}.

\underline{\Cref{as:kfullrank}:} 
Using the moment generating function of the multi-variate normal distribution, 
the RBF kernel \eqref{eq:krbf} can be written
as
\begin{equation}
    K(\x,\x') = \int a(\x,\xib)^* a(\x',\xib)\phi(\xib)\, d \xi,
\end{equation}
where 
\begin{equation}
    a(\x,\xib) = e^{i\xib^*\x},
    \quad
    \phi(\xib) = \frac{\sigma^d}{(2\pi)^{d/2}}
    e^{-\sigma^2\|\xib\|2/2}.
\end{equation}
Thus,
\begin{equation}
    \frac{\partial^2}{\partial \x \partial \x'}
    K(\x,\x') = \int g(\x,\xib)^* g(\x',\xib)\phi(\xib)\, d \xi,
\end{equation}
where
\begin{equation}
    g(\x,\xib) = \xib e^{i\xib^*\x}.
\end{equation}
For any distinct $\xt_j$, $j=1,\ldots,N$,
we have that $g(\xt_j,\xib)$ are linearly independent
functions over $\xib$.  Hence, the matrix $M(\X)$
in \eqref{eq:Ddef} must be full rank.

\underline{Proof of the theorem:}
Since the kernel satisfies
Assumptions~
\ref{as:kscale}---\ref{as:kfullrank},
We can thus apply \Cref{thm:bad_min_gen}
to find a local stable equilibrium 
with finite distance
\eqref{eq:xtfinite}.  
We only have to prove that the distance
scales with $\sigma$.  To this end,
suppose that $\Xt_i^{(1)} = \{\xt_j^{(1)}\}$ is a locally minima of $J_i(\Xt_i)$
for the RBF kernel with $\sigma = \sigma_1$
for some $\sigma_1$.
Then, given any $\sigma_2 > 0$, we can take
a new set of points
\begin{equation}
    \xt_j^{(2)} = \x_i + \frac{\sigma_2}{\sigma_1}
    (\xt_j^{(1)} - \x_i),
\end{equation}
meaning that we simply scale the distances
of the points $\xt_j$ from $\x_i$ by a factor
$\sigma_2 / \sigma_1$.  Then, it is easily 
verified that $\xt_j^{(2)}$ will also
be local minima of $J_i(\Xt_i)$ with 
the RBF kernel with width $\sigma = \sigma_2$.

\section{Proof of \texorpdfstring{\Cref{thm:diverging}}{}}

\begin{lemma}  \label{lem:vfix}
For $\lambda > 0$ sufficiently small, there exists
an $v_0 > 0$ such that
\begin{equation} \label{eq:vfix}
    v_0 = -\eta_g\eta_d \sum_{j=0}^\infty \rho^j\phi'(jv_0), \quad
    \rho = 1 -\eta_d \lambda.
\end{equation}
\end{lemma}

\begin{proof}
Consider the function:
\begin{equation}
    F(v,\rho) = v + \eta_g\eta_d \sum_{j=0}^\infty \rho^j\phi'(jv_0).
\end{equation}
For any $\rho$,
\[
    \lim_{v \rightarrow \infty} F(v,\rho) = 
    \lim_{v \rightarrow \infty} v = \infty.
\]
Also, for $\rho = 1$,
\begin{align}
    \lim_{v \rightarrow 0} F(v,1)
    &= \lim_{v \rightarrow 0}  
     \eta_d \sum_{j=0}^\infty \phi'(jv)
     \nonumber \\
    &= \eta_d \lim_{v \rightarrow 0}  \frac{1}{v} 
    \int_0^\infty \phi'(u)\, du = 
     -\eta_d \lim_{v \rightarrow 0}  \frac{1}{v} 
    \phi(0) = -\infty.
\end{align}
The for $\rho$ sufficiently close to $\rho = 1$,
there must exists a $v$ such that $F(v,\rho) < 0$.
Since, $\lim_{v \rightarrow \infty} F(v,\rho) = \infty$
and there exists a $v$ with $F(v,\rho) < 0$,
and $F(v,\rho)$ is continuous, there must exist
a $v_0$ such that $F(v_0,\rho)=0$.
\end{proof}

Now select any unit vector $\u \in \Real^d$
and initial condition $\xt^0_0$.
Find $v_0 > 0$ as in \Cref{lem:vfix}, 
and define $f^k(\x)$ and $\xt^k_0$ as:
\begin{subequations} \label{eq:fxdiv}
\begin{align} 
    f^{k}(\x) &= -\eta_d\sum_{j=0}^{\infty} \rho^j \phi(\|\x-\xt^0_0-(k-1-j)v_0\u\|),
    \\
    \xt^k &= \xt^0_0 + k v_0 \u.
\end{align}
\end{subequations}
We show that $f^k(\x)$ and $\xt^k_0$ are solutions
to \eqref{eq:update_sing}.
The update for the discriminator is: 
\begin{align}
    f^{k+1}(\x) &=  -\eta_d\sum_{j=0}^{\infty} \rho^j \phi(\|\x-\xt^0_0-(k-j)v_0\u\|)   \nonumber \\
    &= -\eta_d \phi(\|\x-\xt^0_0-kv_0\u\|)
    - \eta_d \sum_{j=1}^{\infty} \rho^j \phi(\|\x-\xt^0_0-(k-j)v_0\u\|) \nonumber \\
    &= -\eta_d \phi(\|\x-\xt^k_0\|)
    - \rho \eta_d \sum_{j=0}^{\infty} \rho^j \phi(\|\x-\xt^0_0-(k-j-1)v_0\u\|) \nonumber \\
    &= -\eta_d K(\x,\xt^k_0)
    + \rho f^k(\x).
\end{align}
Hence, $f^k(\x)$ satisfies the update \eqref{eq:update_singf}.
Also, observe that the gradient of the discriminator in \eqref{eq:fxdiv} is:
\begin{align}
    \nabla f^{k+1}(\xt^k_0) &= 
     -\eta_d\sum_{j=0}^{\infty} \rho^j \frac{\partial}{\partial \x} \left[ \phi(\|\x-\xt^0_0-(k-j)v_0\u\|)\right]_{\x=\xt^k_0}  \nonumber \\
     &= 
     -\eta_d\sum_{j=0}^{\infty} \rho^j \frac{\partial}{\partial \x} \left[ \phi(\|\x-\xt^{k-j}_0\|)\right]_{\x=\xt^k_0}  
     \nonumber \\
     &= 
     -\eta_d\sum_{j=0}^{\infty} \rho^j  \phi'(\|\xt^k_0-\xt^{k-j}_0\|) 
     \frac{(\xt^k_0-\xt^{k-j}_0)}{\|\xt^k_0-\xt^{k-j}_0\|}
     \nonumber \\
      &= -\eta_d\sum_{j=0}^{\infty} \rho^j  \phi'(jv_0)\u = \frac{\eta_d}{\eta_g} v_0 \u, 
\end{align}
where the last step follows from \eqref{eq:vfix}.
Hence, for $\xt^k_0$ defined in 
\eqref{eq:fxdiv}, we have
\begin{align}
    \xt^{k+1}_0 &= \xt^{k}_0 + \eta_g\eta_dv_0 \u
    = \xt^{k}_0 + \eta_g\nabla f^k(\xt^k_0).
\end{align}
Hence, $\xt^k_0$ defined in 
\eqref{eq:fxdiv} satisfies the update
\eqref{eq:update_singx}.

\section{Approximate Isolated Points}
\label{sec:approx}

As stated in \Cref{sec:model}, the isolated assumption 
\eqref{eq:Kdist} may be too strict to achieve exactly
in practice.
In this section, we briefly consider a weaker version
of this assumption.  To state the approximation assumption,
define
\begin{equation}
    g^k(\x) := \frac{\partial f^k(\x)}{\partial \x},
    \quad 
    G(\x,\x') = \frac{\partial K(\x,\x')}{\partial \x}.
\end{equation}
The updates in the local region $V_i$ under the perfect
isolation assumption \eqref{eq:Kdist} can then be written as
\begin{subequations} \label{eq:update_grad}
\begin{align} 
    g^{k+1}(\x) &=g^k(\x) +  \eta_d \rbrac{
     p_i G(\x,\x_i) - 
    \sum_{j \in N_i} \pt_j G(\x,\xt^k_j) - \lambda g^k(\x)  }\\
    \xt^{k+1}_j &= \xt^k_j + \eta_g g^k(\xt^k_j).
\end{align}
\end{subequations}
Now, instead of \eqref{eq:Kdist}, we suppose that there are neighborhoods
$V_i$ around each sample $\x_i$ such that
\begin{equation}  \label{eq:Kdistapprox}
    \|G(\x,\x')\| \leq \epsilon \mbox{ for all } \x \in V_i 
    \mbox{ and } \x' \in V_j \text{ for all $i\neq j$},
\end{equation}
for some $\epsilon \geq 0$.
In this case, we call the set of neighborhoods
$V_i$, \emph{$\epsilon$-isolated neighborhoods}.
Note that under assumption \eqref{eq:Kdist}, the bound \eqref{eq:Kdistapprox} will hold with $\epsilon = 0$.  So, for $\epsilon > 0$, 
\eqref{eq:Kdistapprox} is weaker than \eqref{eq:Kdist}.

Next, \eqref{eq:fup} and \eqref{eq:xtup}
can be written as
\begin{subequations} \label{eq:update_grad}
\begin{align} 
    g^{k+1}(\x) &= g^k(\x) + \eta_d \rbrac{
    \sum_{i=1}^{N_r} p_i G(\x,\x_i) - 
    \sum_{j=1}^{N_g} \pt_j G(\x,\xt^k_j) - \lambda g^k(\x)  } \label{eq:update_grad_g} \\ 
    \xt^{k+1}_j &= \xt^k_j + \eta_g g^k(\xt^k_j).
\end{align}
\end{subequations}
Now, fix a true point $\x_i$.
We can write \eqref{eq:update_grad_g} as
\begin{equation}
    g^{k+1}(\x) = g^k(\x) + \eta_d \rbrac{
     p_i G(\x,\x_i) - 
    \sum_{j \in N_i} \pt_j G(\x,\xt^k_j) - \lambda g^k(\x)  } + \eta_d v^k(\x),
\end{equation}
where $v^K(\x)$ is the term from other neighborhoods:
\begin{equation}
     v^k(\x) := \sum_{k \neq i} p_i G(\x,\x_i) - 
    \sum_{j \not \in N_i} \pt_j G(\x,\xt^k_j).
\end{equation}
From \eqref{eq:Kdistapprox}, for all $\x \in V_i$,
the term $v^k(\x)$
can be bounded as
\begin{equation}
     \|v^k(\x)\| \leq \sum_{k \neq i} p_i \epsilon  
    + \sum_{j \not \in N_i} \pt_j \epsilon \leq 2\epsilon.
\end{equation}
Thus, the local dynamical system in the region $V_i$
is
\begin{subequations} \label{eq:update_approx}
\begin{align} 
    g^{k+1}(\x) &=g^k(\x) +  \eta_d \rbrac{
     p_i G(\x,\x_i) - 
    \sum_{j \in N_i} \pt_j G(\x,\xt^k_j) - \lambda g^k(\x)  } + \eta_d v^k(\x) \\ 
    \xt^{k+1}_j &= \xt^k_j + \eta_g g^k(\xt^k_j).
\end{align}
\end{subequations}
Hence, the system
\eqref{eq:update_approx} is
identical to the local dynamical 
system \eqref{eq:update_grad}, except
for a bounded term $\|v(\x)\| \leq 2\epsilon$.

Now suppose that $\Xt_i^* = \{\xt_j^*\}$ is a locally exponentially 
stable equilibrium point of the system \eqref{eq:update_grad}
under perfect isolation \eqref{eq:Kdistapprox}.
Then, such points will remain stable under the perturbations
by $v^k(\x)$ in \eqref{eq:update_approx}.  For example,
using standard nonlinear systems results in \cite{vidyasagar2002nonlinear},
one can show that, for $\epsilon$ sufficiently small and $\|\xt^0_j - \xt^*_j\|$
sufficiently small, there exists a $C \geq 0$ such that
\begin{equation}
    \|\xt^k_j - \xt^*_j\| \leq C \epsilon \quad \forall j \in N_i, 
    \quad \forall k \geq 0.
\end{equation}
Hence, the solutions of \eqref{eq:update_approx} will remain close
to the equilibrium point.
The constant $C$ will, in general, depend on the eigenvalues of
the linearization.

\section{Potential Solutions via Multi-Scale Kernels}
The above discussion and simulations shows that
there is a fundamental trade-off with respect to the kernel width.
On the one hand, very wide kernels tend to provide
slow rates of convergence.  In addition, they are
not able to accurately discriminate between true
points that are close.  On the other hand, 
very narrow width kernels can result in 
true points being isolated from one another.
The results in this paper show that isolated
points can lead to both approximate mode collapse
and divergence of generated points far away from
the true distribution.

Fortunately, the analysis in the paper suggests
a way to avoid both of these conditions.
Specifically, suppose we consider a kernel 
where the discriminator is of the form:
\begin{equation}
    f(x,\theta_1,\theta_2) :=
    a_1(x)^\intercal \theta_1 + 
    a_2(x)^\intercal \theta_2,
\end{equation}
where $a_1(x)$ and $a_2(x)$ are two basis
functions 
and $\theta=(\theta_1,\theta_2)$
are the parameters.  The kernel for this
discriminator is:
\begin{equation} \label{eq:ktwo}
    K(x,x') = K_1(x,x') + K_2(x,x'),
\end{equation}
where $K_i(x,x') = a_i(x)^\intercal a_i(x')$
are the kernels for each of the bases.

Now, suppose that 
 $a_1(x)$ and $a_2(x)$ are selected so that
 $K_1(x)$ has a wide width and $K_2(x)$ has a small
 width.  We call such a kernel \emph{multi-scale}.
 The overall kernel \eqref{eq:ktwo}
 will have a heavy tail component 
 due to the "wide" kernel
 $K_1(x)$.  Hence, it can avoid the isolated points
 problem.  On the other hand, since $K_2(x)$ has
 a small width, the kernel will have a "sharp"
 component.  This may enable 
 fast convergence near the true distribution.
 
To illustrate the potential use of such a multi-scale
kernel, Fig.~\ref{fig:concat_kernels} 
compares the performance of kernels with 
three fixed widths
of $\sigma \in \{ 0.1,1,10 \}$
 with a single kernel concatenated with all three widths.  The true and generated data are single
 point masses with an initial distance of 10.
As expected, the very low width kernel 
($\sigma=0.1$) fails to converge while the
wide width kernels ($\sigma = 1,10$) converge slowly.
In contrast, the concatenated kernel is able to
get fast convergence in both the initial and later
stages.

Of course, further work will be needed to 
find out the best architectures for networks
with multiple widths for complex practical data.
In the context of neural network discriminators, such
multi-scale kernel behavior may be achieved using parallel networks with varying depths, or the use of skip connections. However, designing such a discriminator requires further investigation.

\begin{SCfigure}
    \includegraphics[width=0.6\linewidth]{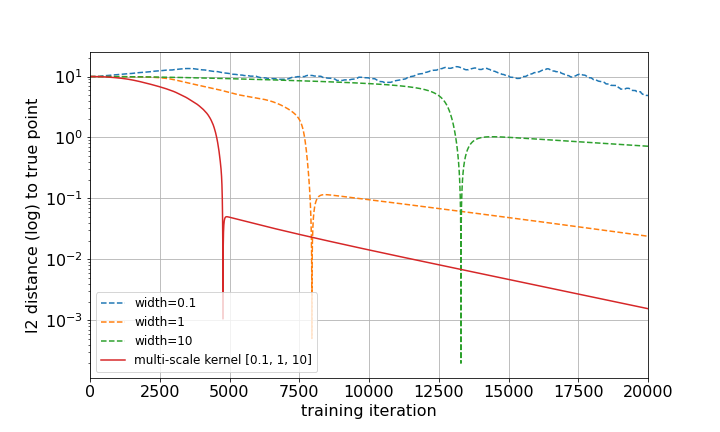}
    \caption{GAN training dynamics using fixed-width kernel discriminators (blue, yellow, and green dashed lines) compared to the GAN training dynamics using a feature map constructed by concatenation of fixed-width feature maps (illustrated by solid line). A concatenated feature map produces a kernel that is a linear combination of its component kernels}
    \label{fig:concat_kernels}
\end{SCfigure}

\section{Experimental Details}

\paragraph{RBF Kernel implementation}
As an approximation to the RBF feature map, we use the following approach detailed in \cite{Rahimi2007RFF}
\begin{equation}
    K(\x,\x')\approx a^\intercal(\x) a(\x'),
\end{equation}
where $a(\x)$ is a basis function vector:
\begin{equation}
    a(\x) = \sqrt{\frac{2}{R}}\begin{bmatrix}\cos(\w_1^\intercal \x)\\ \sin(\w_2^\intercal \x) \\ \vdots \\ \cos(\w_R^\intercal \x)\\ \sin(\w_R^\intercal \x)\end{bmatrix}, \quad \w_i \sim \mathcal{N}(0, \frac{1}{\sigma^2}\mathbf{I}).
\end{equation}
When $R \rightarrow \infty$, $K(\x,\x') \rightarrow e^{-\|\x-\x'\|^2/(2\sigma^2)}$.
In experiments involving the random features RBF kernel, we set $R=1000$. We can see in \cref{fig:rf_approx} that this provides a good estimate of the true kernel.

\begin{figure}
    \centering
    \includegraphics[width=\linewidth]{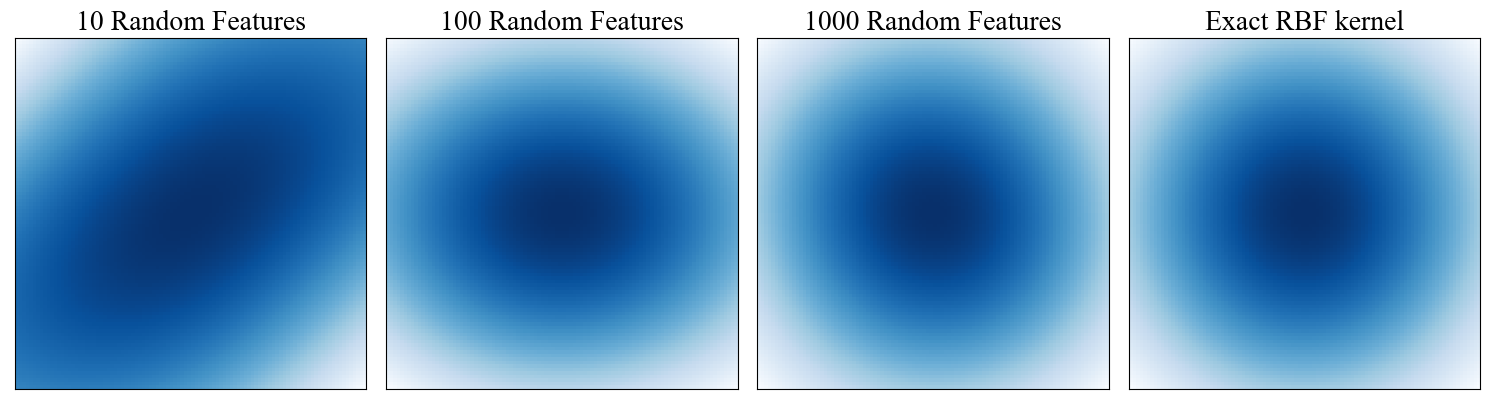}
    \caption{A heat map of different RBF kernel approximations centered at the origin. The exact RBF kernel ($\sigma=1$) is shown on the right.}
    \label{fig:rf_approx}
\end{figure}

\paragraph{Normalized Wasserstein Distance}
In \Cref{fig:failures-modes}, the normalized Wasserstein distance is the ratio
\begin{equation} \label{eq:wassnorm}
    \beta_k := 
    \frac{\|\Prob_g^k-\Prob_r\|_2}{\|\Prob_g^0-\Prob_r\|_2},
\end{equation}
where $\Prob_r$ is the true distribution,
$\Prob_g^k$ is the generated distribution after $k$
iterations and $\|\cdot\|_2$ is the Wasserstein-2 distance.
Hence $\beta_k$ in \eqref{eq:wassnorm} is the change in the distance of the generated
distance to the true distribution relative to the initial distance.
In In \Cref{fig:failures-modes}, we plot the normalized distance after $k=4(10)^4$
iterations.
Note that for discrete distributions, the Wasserstein-2 distance
can be estimated by solving the optimal transport
problem \cite{bonneel2011displacement}.

\paragraph{Neural Network Discriminator} 
While the focus of the paper is on kernel-based discriminators,
here we rerun our two dimensional experiments with a fully connected ReLU network as our discriminator. Layer width is held constant $W=400$, while depth is varied between $L = \{1,2,3,4\}$. We note that in the wide layer (NTK) regime, the corresponding kernel width decreases as the number of layers increase \cite{jacot2018neural}. We set $\lambda=0$, $\eta_d = \eta_g = 10^{-2}$, and use 40k training steps. Just as in earlier experiments, generated points are updated directly according to \cref{eq:xtup}.The discriminator weights are updated by standard backpropagation and gradient descent.

In \Cref{fig:wass_NN_d2}, we see that when we train a GAN with discriminator depths of one and two layers, the Wasserstein distance between distributions only changes by a small amount after training. This failure corresponds nicely to what happens in the large kernel width regime in \Cref{fig:wass-convergence}.
As we increase the number of layers to four, we can think of the effective kernel width of the discriminator decreasing, allowing for individual true points to be differentiated by the discriminator. We see that in this case we get much better convergence behavior. 
\Cref{fig:trajectories_NN_d2} also supports the connection between network depth and effective kernel width. In these trajectory plots we observe large oscillations when the discriminator cannot properly distinguish between true points (catastrophic forgetting) and is far from the isolated points regime.

Lastly, in order to isolate the effect of neural network depth on convergence rate, we train a GAN with an NTK discriminator (following the closed-form solution of \cite{bietti_inductive_2019}) on a single generated point and true point. These points are initialized at a distance of 0.1 in two dimensions, and $\lambda=0.1$, $\eta_g=\eta_r=10^{-3}$. In \cref{fig:rate_NN_d1}, we see that by increasing the number of layers we see an increase in convergence rate.

\begin{figure}
    \centering
    \includegraphics[width=0.5\linewidth]{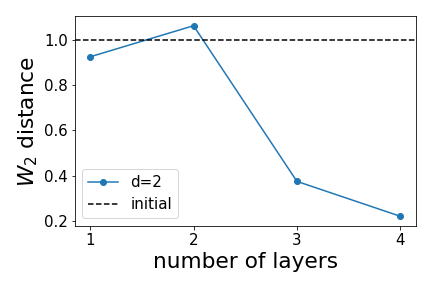}
    \caption{Median change in Wasserstein distance between true and generated distributions after 40k iterations in a two dimensional setting. Fully connected ReLU networks of different depths are used as discriminators}
    \label{fig:wass_NN_d2}
\end{figure}

\begin{figure}
    \centering
    \includegraphics[width=0.8\linewidth]{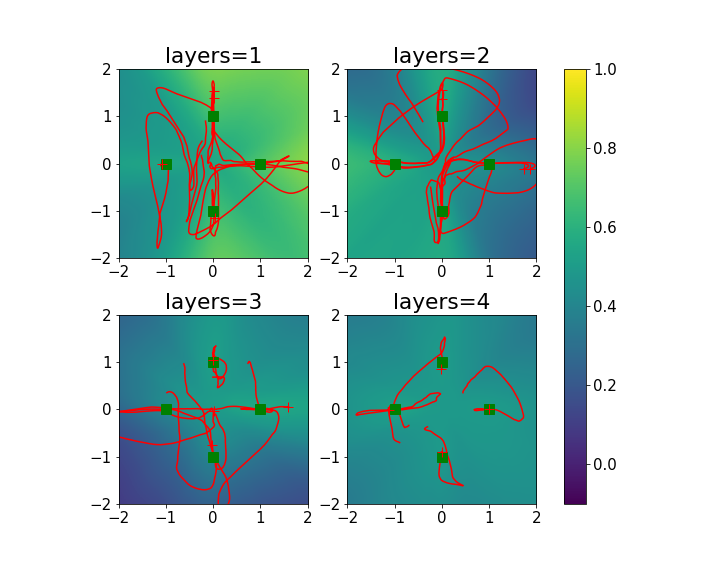}
    \caption{Behavior of joint GAN training with discriminator depth. Example trajectories of generated points over the course of training (red lines with final point marked as a cross), true distribution (green), final discriminator (blue and yellow colormap)}
    \label{fig:trajectories_NN_d2}
\end{figure}

\begin{figure}
    \centering
    \includegraphics[width=0.8\linewidth]{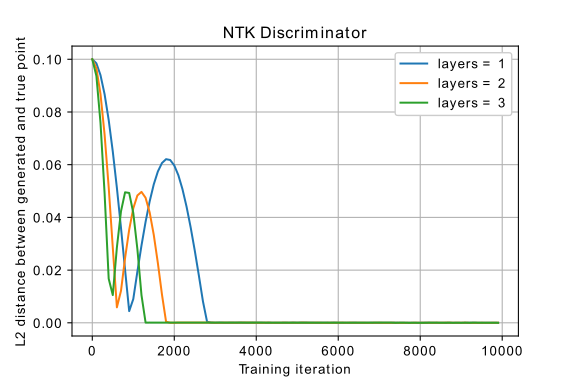}
    \caption{Convergence rates of joint GAN training under increasing NTK discriminator depth. }
    \label{fig:rate_NN_d1}
\end{figure}

\end{document}